\documentclass{article}

\usepackage{microtype}
\usepackage{graphicx}
\usepackage{subcaption}
\usepackage{booktabs} %
\usepackage{titletoc}
\usepackage{makecell}
\usepackage{placeins}

\usepackage{hyperref}
\usepackage[normalem]{ulem}
\usepackage[dvipsnames]{xcolor}
\usepackage{placeins}

\usepackage[accepted]{icml2026}

\usepackage{amsmath}
\usepackage{amssymb}
\usepackage{mathtools}
\usepackage{amsthm}

\usepackage[capitalize,noabbrev]{cleveref}

\theoremstyle{plain}
\newtheorem{theorem}{Theorem}[section]
\newtheorem{proposition}[theorem]{Proposition}
\newtheorem{lemma}[theorem]{Lemma}
\newtheorem{corollary}[theorem]{Corollary}
\theoremstyle{definition}

\newtheorem{assumption}[theorem]{Assumption}
\theoremstyle{remark}
\newtheorem{remark}[theorem]{Remark}

\usepackage[textsize=tiny]{todonotes}

\icmltitlerunning{FIRE: Multi-Fidelity Regression with Distribution-Conditioned In-Context Learning Using Tabular Foundation Models}

\begin{document}

\twocolumn[
  \icmltitle{
  FIRE: Multi-Fidelity Regression with Distribution-Conditioned In-Context Learning Using Tabular Foundation Models
  }

  \icmlsetsymbol{equal}{*}

  \begin{icmlauthorlist}
    \icmlauthor{Rosen Ting-Ying Yu}{comp}
    \icmlauthor{Nicholas Sung}{yyy}
    \icmlauthor{Faez Ahmed}{comp,yyy}
  \end{icmlauthorlist}

  \icmlaffiliation{yyy}{Department of Mechanical Engineering, Massachusetts Institute of Technology, Cambridge, MA, USA}
  \icmlaffiliation{comp}{Center for Computational Science and Engineering, Massachusetts Institute of Technology, Cambridge, MA, USA}

  \icmlcorrespondingauthor{Rosen Ting-Ying Yu}{rosenyu@mit.edu}

  \icmlkeywords{Machine Learning, ICML}

  \vskip 0.3in
]

\printAffiliationsAndNotice{}  %

\begin{abstract}
  Multi-fidelity (MF) regression often operates in regimes of extreme data imbalance, where the commonly-used Gaussian-process (GP) surrogates struggle with cubic scaling costs and overfit to sparse high-fidelity observations, limiting efficiency and generalization in real-world applications. We introduce FIRE, a training-free MF framework that couples tabular foundation models (TFMs) to perform zero-shot in-context Bayesian inference via a high-fidelity correction model conditioned on the low-fidelity model's posterior predictive distributions. This cross-fidelity information transfer via distributional summaries captures heteroscedastic errors, enabling robust residual learning without model retraining. Across 31 benchmark problems spanning synthetic and real-world tasks (e.g., DrivAerNet, LCBench), FIRE delivers a stronger performance–time trade-off than seven state-of-the-art GP-based or deep learning MF regression methods, ranking highest in accuracy and uncertainty quantification with runtime advantages. Limitations include context window constraints and dependence on the quality of the pre-trained TFMs. Code \& data can be found here: \href{https://github.com/rosenyu304/FIRE}{https://github.com/rosenyu304/FIRE}.
\end{abstract}

\section{Introduction}

Multi-fidelity (MF) regression is the problem of learning to predict an expensive, high-accuracy target function using a small number of high-fidelity observations together with a larger set of cheaper but less accurate low-fidelity (LF) observations.
Multi-fidelity regression typically leverages probability models that couple high-fidelity (HF) measurements with abundant, noisy low-fidelity data to generate accurate predictions and uncertainty quantification, which are critical in various fields. For instance, engineering and scientific modeling increasingly rely on MF regression to navigate the trade-off between expensive HF simulation (e.g., physical experiments or time-consuming computational fluid dynamics CFD simulations) and cheap LF analytical parametric modeling with lower accuracy \citep{bonfiglio2018multi,feldstein2020multifidelity,romor2023multi,zhang2024transfer}. Another application is hyperparameter optimization (HPO) tasks, where training budgets (epochs) serve as varying fidelities \citep{eggensperger2021hpobench,wu2020practical,rakotoarison2024ifbo}. However, these practical applications often reside in an extreme data imbalance regime, where HF measurements are so scarce (less than 10\% of the total dataset), that common probability regression models struggle to maintain reliable generalization \citep{xing2023continuar,Qian2024,aretz2025multifidelity}. 

Classical MF regression methods typically rely on Gaussian Process (GP) surrogates using autoregressive (AR) couplings and residual decompositions to transfer knowledge across fidelities \citep{Kennedy2000,Perdikaris2017,xing2021residual}. However, these methods face significant computational and statistical barriers. GP's cubic scaling costs make it less effective as LF datasets grow \citep{fernandez2016review,liu2020gaussian}. GP also is prone to overfitting when learning complex kernel hyperparameters from very few HF data \citep{cutajar2019deep}. Furthermore, multi-fidelity data are rarely collected on nested input sets in practice. HF measurements are often obtained independently, resulting in disjoint input locations across fidelities~\citep{niu24MFRNP}. Most GP-based MF methods rely on nested designs ($X_{\text{HF}} \subseteq X_{\text{LF}}$) to enable stable learning~\citep{Kennedy2000,xing2023continuar}. When this assumption is violated, these models struggle to transfer information reliably across fidelities.

Recent advances in tabular foundation models (TFMs) provide an effective alternative to classification and regression tasks: zero-shot Bayesian inference that uses in-context learning (ICL) to generate posterior predictions in a single forward pass without local parameter updates \citep{muller2021transformers}. Pre-trained offline on millions of synthetic structured causal priors, TFMs such as TabPFN \citep{hollmann2023tabpfn,hollmann2025accurate} are shown to have an advantage over existing tabular learning models in data-scarce regimes for many regression tasks \citep{ye2025closer} and rank highly on tabular benchmarks \citep{erickson2025tabarena}. These findings motivate our investigation into whether TFMs 
can address the limitations of existing MF methods and provide state-of-the-art (SOTA) performance under very limited non-nested HF data.

Beyond scalability and model choice, a key shortcoming of many MF methods using multiple surrogates lies in their transfer mechanism. Autoregressive and residual-based formulations typically condition HF corrections on LF mean predictions, assuming that cross-fidelity discrepancies are homoscedastic \citep{xing2021residual,Ravi2024}. However, in realistic settings, LF uncertainty varies strongly across the input space, and residual errors tend to grow in precisely those regions where LF predictions are least certain \citep{sella2023projection,wilke2024multifidelity}. Providing predictive variance and quantiles as explicit features may help
a constrained HF correction model use this structure, motivating our
empirical study of distribution-conditioned residual learning \citep{zou2008composite,yao2018using,song2019distribution}.

We introduce FIRE (Fidelity‑aware In-context REgression), a training-free framework for MF regression using TFMs with uncertainty-aware autoregressive residual learning. FIRE decomposes multi-fidelity regression into three stages: (i) low-fidelity inference, (ii) residual modeling conditioned on augmented statistical information, and (iii) additive uncertainty propagation, using frozen TFMs for zero-shot regression. To summarize, our contributions include:
\begin{enumerate}
\item \textbf{Algorithmic:} We introduce distribution-conditioned residual transfer for MF regression: a two-stage inference procedure that conditions the HF correction on LF predictive distributions (mean/variance/quantiles), enabling heteroscedastic and non-Gaussian discrepancy handling without retraining.

\item \textbf{Systems:} We show how to turn an off-the-shelf TFM into a training-free MF surrogate via in-context residual learning and bi-level fidelity aggregation that scales across multiple fidelities.

\item \textbf{Empirical:} On 31 MF benchmarks, FIRE improves the accuracy–uncertainty–runtime tradeoff and remains robust in extreme HF scarcity (2–5\%), where standard GP-based methods struggle. Overall, it achieves the best predictive performance for both non-nested and nested problems.

\item \textbf{Datasets:} We collect a comprehensive benchmark of 31 problems, including HPO and engineering tasks, while adding a new CFD-based car simulations dataset for evaluating MF regression generalization, moving beyond the limited test cases from previous works.

\end{enumerate}

\section{Background}\label{Section:Background}

\textbf{Multi-fidelity Regression.} Multi-fidelity regression problem is often defined over an input domain $\mathcal{X} \subset \mathbb{R}^d$, with a hierarchy of $T$ information sources (fidelities), indexed by $t \in \{1, \dots, T\}$. Each fidelity corresponds to an unknown function $f^{(t)}: \mathcal{X} \to \mathbb{R}$, with increasing accuracy and cost as $t$ increases. For each fidelity $t$, we observe a dataset
\begin{equation}\label{eq:MF_data}
    \mathcal{D}_t = \{(x_i^{(t)}, y_i^{(t)})\}_{i=1}^{N_t}, \quad y_i^{(t)} = f^{(t)}(x_i^{(t)}).
\end{equation}
In practice, the standard MF imbalance regime is: $N_T \ll N_{T-1} \ll \dots \ll N_1,$ reflecting the high cost of acquiring HF observations \citep{Qian2024}. The objective is to construct a predictor for the HF response $f^{(T)}(x)$ by leveraging information from all available lower-fidelity data. 

MF regression algorithms commonly adopt an additive autoregressive formulation \citep{Kennedy2000},
\begin{equation}\label{eq:AR_formulation}
    f^{(t)}(x) = \rho_t f^{(t-1)}(x) + \delta_t(x),\quad t\in{2, ..., T}
\end{equation}%
where $\rho_t \in \mathbb{R}$ is a scaling factor and $\delta_t(x)$ is a residual process representing cross-fidelity discrepancies between fidelity levels $t$. This formulation is commonly adopted by classical GP-based methods to transfer information across fidelities \citep{Perdikaris2017,xing2021residual}. However, these approaches typically assume nested input, and their performance degrades in non-nested regimes where HF data are sparse and misaligned \citep{wang2022gar}. Motivated by this challenge, recent work proposes generalized autoregressive formulations that support non-nested data in high-dimensional or infinite-fidelity settings \citep{xing2023continuar}. Other lines of work explore deep learning models for MF regression to improve scalability via learned representations \citep {Yi2024MFBNN,niu24MFRNP,taghizadeh2025multifidelity}. In this work, we uses zero-shot TFM and distribution conditioning to improve upon previous avenue.

\textbf{Tabular Foundation Models (TFMs).} TFMs have demonstrated the ability to achieve strong regression performance by leveraging in-context learning~\citep{erickson2025tabarena}, in which predictions for query inputs are produced by conditioning on a set of context datasets. While FIRE is TFM agnostic, to demonstrate its efficacy, we leverage one of the current SOTA TFM models on tabular regression benchmarks, Prior-data Fitted Networks, and especially the TabPFN model~\citep{hollmann2023tabpfn,hollmann2025accurate,grinsztajn2025tabpfn25advancingstateart}. TabPFN is a transformer-based model pre-trained on millions of pre-trained on synthetic datasets $D_{\text{train}}=\{(x_i,y_i)\}_{i=1}^M$ from randomly generated structural causal models by maximizing the likelihood $q(y | x,\, D_{\text{train}})$~\citep{hollmann2025accurate}. Given an in-context dataset ${D}_{\text{context}}$ during inference, TabPFN estimates the posterior predictive distribution of the target $y_{*}$ at the input location $x_{*}$ as $p(y_{*}|x_{*}) \approx q_\theta(y_{*} \mid x_{*}, \mathcal{D}_{\text{context}})$.

Transformer-based architectures are particularly well-suited to structured inputs, including discrete variables. Recent work has exploited this property to encode fidelity as an input token in TFMs for multi-fidelity HPO~\citep{lee2025costsensitive,rakotoarison2024ifbo} and shown promising result. The results motivate the use of fidelity token without architectural modification in MF regression setting beyond HPO. However, as we show later in Appendix \ref{Section:Ablation}, the token is not sufficient in imbalanced data settings.

\textbf{Heteroscedasticity and Distributional Stacking. }Heteroscedasticity refers to input-dependent variability, where predictive uncertainty varies across the domain and cannot be captured by a constant noise assumption~\citep{kersting2007most,lazaro2011variational}. In multi-fidelity regression, such effects arise in cross-fidelity discrepancies: regions where LF models are less reliable often exhibit larger and more variable residuals~\citep{sella2023projection,wilke2024multifidelity}. An explicit LF predictive-variance feature may help a constrained residual
learner use spatial variation in LF uncertainty. Its omission does not,
by itself, imply that a model assumes uniform LF reliability. While weighted architectures and warping have been proposed to model input-dependent noise, they frequently rely on Gaussian assumptions that fail to capture skewed or heavy-tailed distributions common in complex real-world systems~\citep{colombo2025warped}. A few results on predictive stacking indicate that leveraging full predictive distributions is more robust than aggregating point estimates when models are imperfect~\citep{yao2018using}. This motivates our model to incorporate conditioning HF corrections on distributional summaries (mean, variance, and quantiles) rather than on the low-fidelity mean alone~\citep{angrist2006quantile,zou2008composite}.

\section{Proposed Method: FIRE}\label{Section:FIRE}

FIRE decomposes zero-shot multi-fidelity regression learning into three stages: (1) low-fidelity inference, (2) distribution-conditioned residual learning, and (3) prediction and uncertainty quantification. This architecture has a few similarities to autoregressive MF decompositions, but avoids explicit parametric assumptions and repeated retraining by using ICL. Instead, all stages operate in a training-free, inference-time manner using pre-trained TFMs. An overview of the FIRE algorithm is shown in Figure \ref{fig:main_figure} and Algorithm \ref{algorithm:1}.

\begin{figure*}[ht!]
    \centering
    \includegraphics[width=.95\linewidth]{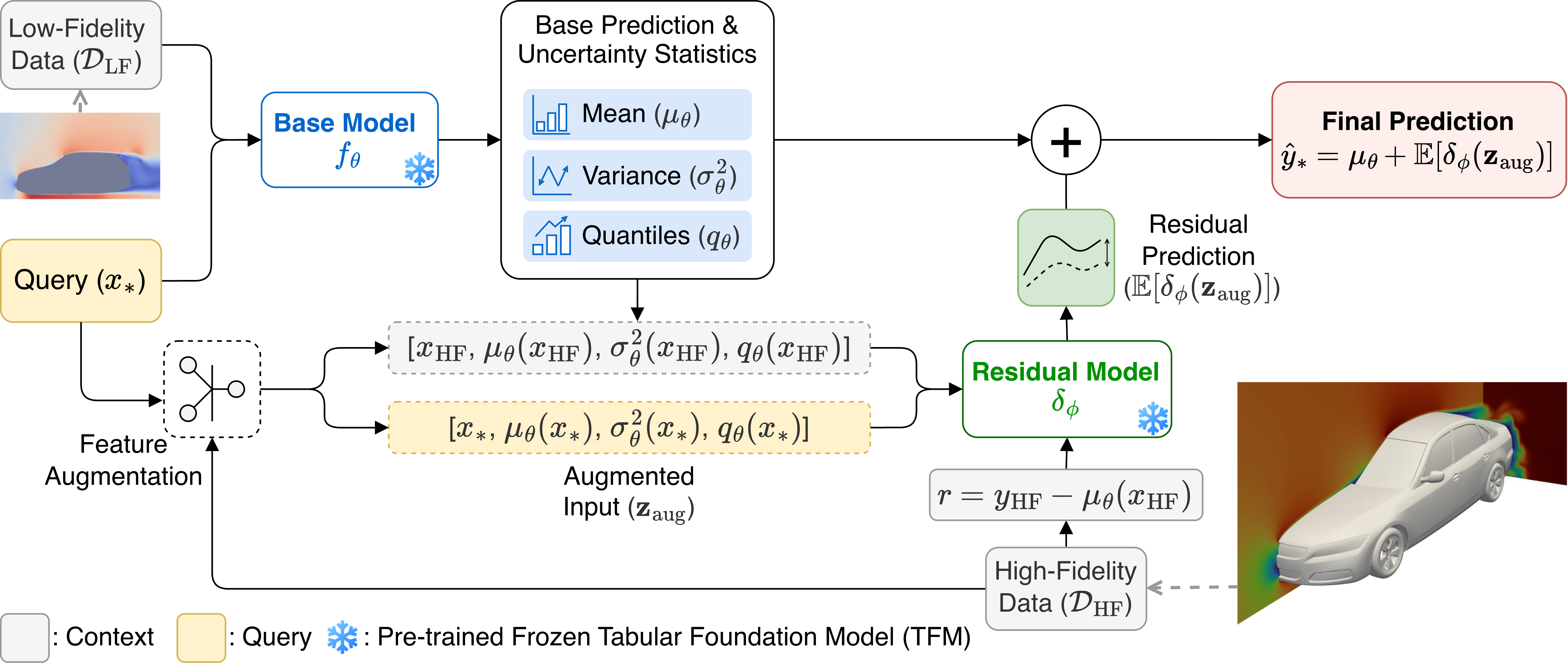}
    \caption{\textbf{Overview of FIRE.} The method decomposes multi-fidelity regression into a low-fidelity base inference and a residual correction step, conditioned on the base model's statistical and distributional information.}
    \label{fig:main_figure}
\end{figure*}

\begin{algorithm}[ht!]
  \caption{FIRE}
  \label{algorithm:1}
  \begin{algorithmic}[1]
    \STATE {\bfseries Input:} \\
     $\quad${\color{Black}Training data:} \\
     $\quad$$\mathcal{D}_{\text{LF}} $ = $ \bigcup_{t=1}^{T-1} \{([{x}_i, t], y_i)\}_{i=1}^{N_t} $ = $ \{(\mathbf{x}_{\text{LF}}, \mathbf{y}_{\text{LF}})\}$ ; \\ $\quad$$\mathcal{D}_{\text{HF}}$ = $\{([x_j, T], y_j)\}_{j=1}^{N_{\text{HF}}}$ = $\{(\mathbf{x}_{\text{HF}}, \mathbf{y}_{\text{HF}})\}$; \\ %
    $\quad${\color{Black}Testing data:} $\mathbf{x}_{*}$ = $\{[x_p, T]\}_{p=1}^{N_{\text{test}}}$ \\
    $\quad${\color{Black}Two TFM models:} $f_{\theta}$ base, $\delta_{\phi}$ residual. \\
    $\quad${\color{Black}Quantile levels:} $\mathcal{Q} = \{0.1, 0.2, \dots , 0.9\}$; \\

    \STATE {\bfseries
    \color{Blue}In-context Learning Phase:}
    
    \STATE $\quad$Fit LF base model $f_{\theta}$ on $\mathcal{D}_{\text{LF}}$

    \STATE $\quad$$\mu_{\theta}(\mathbf{x}_{\text{HF}}) \leftarrow \mathbb{E}[f_{\theta}(\mathbf{x}_{\text{HF}})]$;\\ 
    $\quad$$\sigma_{\theta}^2(\mathbf{x}_{\text{HF}}) \leftarrow \text{Var}[f_{\theta}(\mathbf{x}_{\text{HF}})]$;\\ 
    $\quad$$q_{\theta}^{(\tau)}(\mathbf{x}_{\text{HF}}) \leftarrow \text{Quantile}(f_{\theta}(\mathbf{x}_{\text{HF}}), \tau)$ for $\tau \in \mathcal{Q}$
    \STATE $\quad$Statistical augmentation: \\
    $\quad$${z}_{\text{aug}} \leftarrow [\mathbf{x}_{\text{HF}},\mu_{\theta}(\mathbf{x}_{\text{HF}}), \sigma_{\theta}^2(\mathbf{x}_{\text{HF}}), \{q_{\theta}^{(\tau)}(\mathbf{x}_{\text{HF}})\}_{\tau \in \mathcal{Q}}]$
    \STATE $\quad$Compute residual: $r \leftarrow \mathbf{y}_{\text{HF}} - \mu_{\theta}(\mathbf{x}_{\text{HF}})$
    \STATE $\quad$$\mathcal{D}_{\text{aug}} \leftarrow  \{({z}_{\text{aug}}, r)\}$
    \STATE $\quad$Fit residual model $\delta_{\phi}$ conditioned on $\mathcal{D}_{\text{aug}}$ \\%to predict residual $r$ given augmented input ${z}_{\text{aug}}$. \\

    \STATE {\bfseries \color{Blue}Prediction Phase:}
    \STATE $\quad$$\mathbf{z}_{\text{aug}}^* \leftarrow [\mathbf{x}_{*}, \mu_{\theta}(\mathbf{x}_{*}),\sigma_{\theta}^2(\mathbf{x}_{*}), \{q_{\theta}^{(\tau)}(\mathbf{x}_{*})\}_{\tau \in \mathcal{Q}}]$
    \STATE $\quad$$\mu_{\phi}(\mathbf{z}_{\text{aug}}^*) \leftarrow \mathbb{E}[\delta_{\phi}(\mathbf{z}_{\text{aug},*})] ; $\\
    $\quad$$\sigma_{\phi}^2(\mathbf{z}_{\text{aug}}^*) \leftarrow \text{Var}[\delta_{\phi}(\mathbf{z}_{\text{aug}}^*)]$;
    \STATE {\bfseries Return Output:} \\
    $\quad$$\hat{y}(\mathbf{x}_*) \gets \mu_{\theta}(\mathbf{x}_{*}) + \mu_{\phi}(\mathbf{z}_{\text{aug}}^*)$;  \\
    $\quad$$\hat{\sigma^2_*} =\sigma^2_\theta(\mathbf{x}_*) + \sigma^2_\phi(\mathbf{z}_{\text{aug}}^*)$
  \end{algorithmic}
\end{algorithm}

\subsection{Bi-Level Fidelity Representation}
Building upon the general $T$-fidelity hierarchy introduced in Section~\ref{Section:Background}, FIRE employs a unified two-level representation to facilitate in-context learning. We aggregate all available lower-fidelity sources $t \in \{1, \dots, T-1\}$ into a singular LF dataset, 
$$\mathcal{D}_{\text{LF}} = \bigcup_{t=1}^{T-1} \{([x_i^{(t)}, t], y_i^{(t)})\}_{i=1}^{N_t}. 
$$
Here, the fidelity index $t$ is treated as a integer categorical feature, allowing the TFM to learn a shared representation across diverse noise regimes. Specifically, we treat the highest fidelity $T$ as the target
(high-fidelity, HF) level at which predictions are ultimately evaluated:
$$\mathcal{D}_{\text{HF}} = \{([x_j^{(T)}, T], y_j^{(T)})\}_{j=1}^{N_T}.$$
This reduction reflects the practical setting in which multiple auxiliary fidelities are available, but only the highest fidelity is queried at test time, exploiting TFM's capability for discrete data modeling in a single forward pass. Empirical evidance on the advantage of our bi-level model over the standard recursive one surrogate model per fidelity level is detailed in Appendix~\ref{appendix:surrogate_level}.

\subsection{Low-Fidelity Inference}\label{Section3.2:Low-Fidelity Inference}

Let $f_\theta$ denote a probabilistic TFM pre-trained offline on a large collection of synthetic regression tasks with frozen weight $\theta$. At inference time, $f_\theta$ performs in-context learning on the aggregated low-fidelity dataset $\mathcal{D}_{\text{LF}}$. Given a high-fidelity input $x_{\text{HF}}$, the model outputs a predictive distribution
$$f_\theta(x_{\text{HF}}) \sim q_\theta(y \mid x_{\text{HF}}, \mathcal{D}_{\text{LF}}),$$
which we summarize using: (1) the predictive mean $\mu_\theta(x)$, (2) the predictive variance $\sigma_\theta^2(x)$, and (3) a finite set of conditional quantiles $\{q_\theta^{(\tau)}(x)\}_{\tau \in \mathcal{Q}}$, where $\mathcal{Q} = \{0.1, 0.2, \dots, 0.9\}$ (TabPFN's default quantile set). These quantities serve as lightweight summaries of the low-fidelity
predictive distribution. We use them as predictive features, but we do not
assume that their uncertainty is exclusively epistemic. %

\subsection{Distribution-Conditioned Residual Learning}

FIRE adopts the autoregressive structure in equation implicitly by learning a residual correction at the highest fidelity level. For each high-fidelity training point $x_{\text{HF}}$, we define the residual$$r = y_{\text{HF}} - \mu_\theta(x_{\text{HF}}).$$
To model this residual, we construct an augmented input vector
\begin{align*}
    z_{\text{aug}} &= [x_{\text{HF}}, \mu_\theta(x_{\text{HF}}), \sigma_\theta^2(x_{\text{HF}}), \{q_\theta^{(\tau)}(x_{\text{HF}})\}_{\tau \in \mathcal{Q}}].
\end{align*}
This augmentation encodes: (1) a point prediction of the lower-fidelity signal, (2) uncertainty magnitude through predictive variance, (3) distributional shape through predictive quantiles. The augmented residual dataset is
$$\mathcal{D}_{\text{aug}} = \{(z_{\text{aug}}, r)\}.$$
A second TFM $\delta_\phi$ (with same weight $\phi \equiv \theta$) is then used to perform in-context learning on $\mathcal{D}_{\text{aug}}$, producing a predictive distribution for the residual:$$\delta_\phi(z_{\text{aug}}) \sim q_\phi(r \mid z_{\text{aug}}, \mathcal{D}_{\text{aug}}).$$

\subsection{Prediction and Uncertainty Quantification}\label{Section3.4:Prediction and Uncertainty Quantification}

Given a test input $x_*$ at the highest fidelity, prediction proceeds as follows:
\begin{enumerate}
    \item Low-fidelity inference: Evaluate the low-fidelity model conditioned on $\mathcal{D}_{\text{LF}}$ to obtain$$\mu_\theta(x_*), \quad \sigma_\theta^2(x_*), \quad \{q_\theta^{(\tau)}(x_*)\}_{\tau \in \mathcal{Q}}.$$
    \item Feature augmentation: Construct the augmented input$$z_{\text{aug}}^* = [x_*, \mu_\theta(x_*), \sigma_\theta^2(x_*), \{q_\theta^{(\tau)}(x_*)\}_{\tau \in \mathcal{Q}}].$$
    \item Residual correction: Predict the residual using $\delta_\phi$ and form the final prediction and uncertainty$$\hat{y}_* = \mu_\theta(x_*) + \mathbb{E}[\delta_\phi(z_{\text{aug}}^*)],\quad \hat{\sigma}_*^2 = \sigma_\theta^2(x_*) + \sigma_\phi^2(z_{\text{aug}}^*)$$
where $\sigma_\phi^2$ denotes the residual model's predictive variance.
The additive variance is a working uncertainty approximation.
Appendix~\ref{appendix:theory:Variance Decomposition for Two-Stage Uncertainty Quantification} gives a working-distribution interpretation (it does not establish exact posterior uncertainty or calibration).
\end{enumerate}

\section{Experiment}

We compare FIRE against seven SOTA GP-based multi-fidelity baselines on a diverse benchmark suite of 31 multi-fidelity
problems that can be categorized into three distinct domains, testing each algorithm's generalizability. Details on the benchmark datasets and algorithm implementation are in Appendices~\ref {Appendix:Section:Benchmark} and~\ref{Appendix:Section:Algorithm_Implementation}.

\subsection{Algorithms}

\textbf{Baseline Algorithms.} We compare against 7 MF regression baselines. The standard GP-based autoregressive methods \textbf{AR(1)} \citep{Kennedy2000}, \textbf{ResGP} \citep{xing2021residual}, \textbf{NARGP} \citep{Perdikaris2017}, infinite-fidelity methods \textbf{ContinuAR} \cite{xing2023continuar}, and multi-information source \textbf{MisoKG} \citep{poloczek2017multi}. We also include deep learning baselines \textbf{MFBNN} (DNN-LR-BNN) \citep{Yi2024MFBNN} and \textbf{MFRNP} \citep{niu24MFRNP}. 

\textbf{FIRE's TFM.} We select \textbf{TabPFNv2.5} \citep{grinsztajn2025tabpfn25advancingstateart} as the model for FIRE due to its top-tier performance in the tabular benchmark even without tuning. We assess FIRE's generalization with different TFMs Appendix \ref{Appendix:Different_TFM_ablation}.

\subsection{Benchmark Datasets Design}\label{Section:5.2:Benchmark Datasets Design}
\textbf{Synthetic Benchmarks (18 tasks).} We include 11 standard two-fidelity MF benchmarks from the MF2 \citep{vanRijn2020MF2}, two three-fidelity problems from Emukit \citep{emukit2023}, and five high-dimensional synthetic tasks (HD suite) with $d\in{10,20,30,40,50}$ input feature dimension \citep{Yi2024MFBNN}, which are used specifically for our algorithm runtime analysis in Appendix~\ref{Appendix:full_results}.

\textbf{HPO Datasets (six tasks).} We take six learning curves datasets from LCBench \citep{zimmer2021auto}, treating epoch budget as fidelity. We implement them as five-fidelity regression tasks, predicting validation performance at the maximum epoch from partial training curves.

\textbf{Physics-based Simulations (seven tasks).} 
We evaluate on seven physics-based multi-fidelity regression tasks.
We take three analytical parametric MF benchmarks from the engineering literature~\citep{Yousefpour2024GPPlus}, and select two Blended Wing Body (BWB) aerodynamic simulated datasets~\citep{Sung2025BlendedNet,Sung2026BlendedNetMF}. 
Finally, we create an additional physics-based MF dataset by augmenting the HF CFD data in 3D car aerodynamic prediction using the DrivAerNet dataset~\citep{elrefaie2025drivaernet} with a corresponding LF model, and also run the LF modeling for the HF concrete compressive strength benchmark \citep{Yeh1998Concrete}.

\textbf{Data Imbalance Protocol.} A critical challenge in multi-fidelity learning is the extreme scarcity of HF data relative to LF data. We systematically evaluate this regime by fixing the LF budget $N_{\text{LF}}$ (set to 200 or 2000 depending on dimensionality $d$) and varying the HF budget $N_{\text{HF}}$ across six imbalance ratios: $\{2\%, 4\%, 5\%, 10\%, 20\%, 25\%\}$ of $N_{\text{LF}}$. This stress-tests the ability of models to effectively transfer information from abundant low-fidelity sources.

\subsection{Experiment Protocol \& Evaluation Metrics}\label{Sectoin:Experiments:Metrics}

\textbf{Experiment Protocol. }
The algorithm evaluation aims to thoroughly compare FIRE to current SOTA MF regression algorithms across problems. For each test problem, our experiment consists of 10 independent trials repeated for a 5-fold outer cross-validation, each trial using a distinct random seed fixed across algorithms. To stress-test model flexibility, we enforce non-nested splits between fidelities. HF points are not co-located with LF points. This reflects realistic scenarios where data sources are independent.

Following~\citet{xing2023continuar}'s protocol of fair benchmarking MF algorithms, all GP-based models are trained with 200 iterations, and all deep learning models with 1000 iterations to reach convergence. All algorithms are implemented in PyTorch and can be GPU-accelerated. All models are used out-of-the-box, without fine-tuning, for a fair comparison of surrogates' generalization across different problems.

Experiments were conducted on a distributed cluster featuring Intel Xeon Platinum 8480+ CPUs and NVIDIA H100 GPUs. To ensure fairness, each run uses identical resources: a single H100 node with 24 CPU cores and 200GB RAM.

\textbf{Individual Metrics. }Following existing literature~\citep{cutajar2019deep,Ravi2024,niu24MFRNP}, we use normalized root mean squared error (NRMSE) and negative log likelihood (NLL) for evaluating accuracy and uncertainty quantification ability, respectively, for our main results. In addition, we measure the wall-clock runtime of each algorithm to compare the computational efficiency.

\textbf{Aggregated Metrics. }Here we evaluate models using the Elo rating~\cite{elo1967proposed} (higher is better), a pairwise comparisons rating that reflects its predicted probability of winning against other models, and average ranking (lower is better). We calibrate 1000 Elo to the performance of MFRNP and perform 100 rounds of bootstrapping to obtain 95\% confidence intervals, similar to~\citet {erickson2025tabarena}'s approach. Additional metrics evaluations (e.g. normalized score, statitstical ranking) is presented in Appendix~\ref{Appendix:full_results}.

\section{Results}\label{Section:Result}

\begin{figure}[ht]
    \centering
    \includegraphics[width=.87\linewidth]{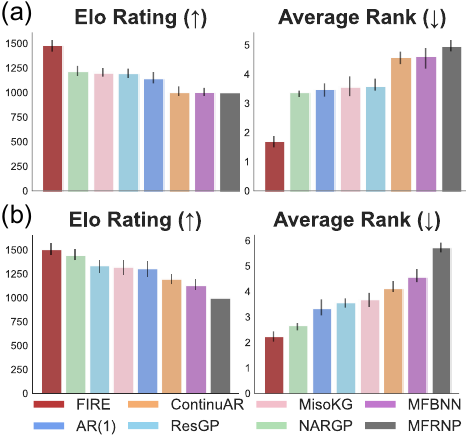}
    \caption{\textbf{Overall performance.} Elo rating and average rank for (a) accuracy performance (NRMSE) and (b) uncertainty calibration (NLL) across 31 problems. FIRE outperforms the other baselines in both accuracy and uncertainty.
    }
    \label{fig:main_barplot}
\end{figure}

\textbf{Overall Performance across Benchmarks.} Figure \ref{fig:main_barplot} summarizes performance across all 31 benchmarks. FIRE achieves the highest Elo rating and lowest average rank of both accuracy (NRMSE) and uncertainty (NLL), consistently outperforming both GP-based and deep learning baselines. This performance is driven by our distributional conditioning strategy: as shown in our ablation study (Figure ~\ref{fig:ablation_each_FIRE_component}), removing the variance or quantile features from the residual model drops $\approx$ 250 Elo rating. This gap highlights that standard mean-only residual couplings used by GP-based methods are insufficient for capturing the complex, input-dependent error structures in these tasks.

\begin{figure}[ht]
    \includegraphics[width=1\linewidth]{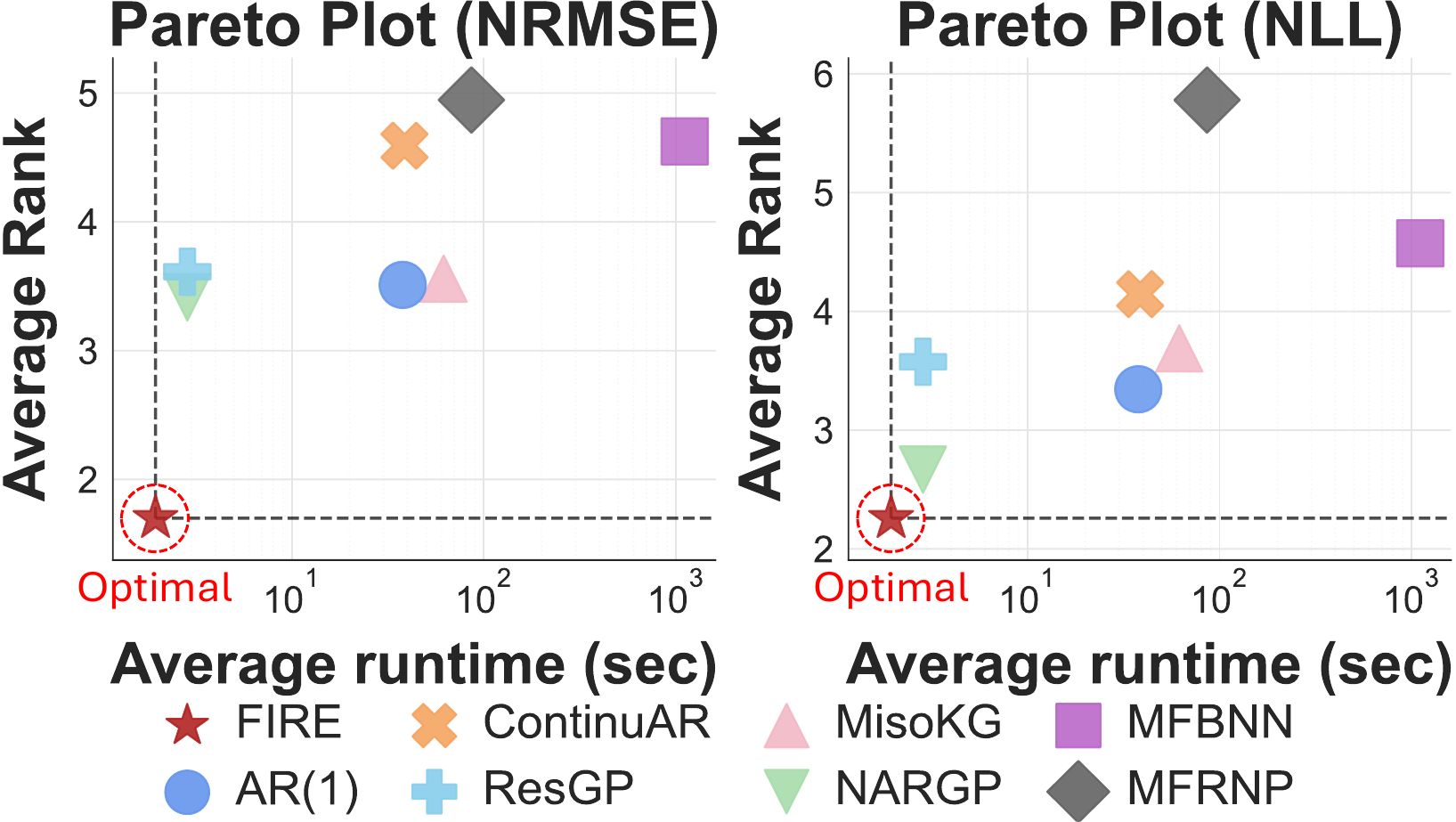}
    \caption{\textbf{Pareto frontier analysis.} Pareto plots of accuracy (Left) and uncertainty (Right) vs. runtime. FIRE dominates the Pareto frontier for both metrics, being the fastest and most accurate. }
    \label{fig:main_pareto}
\end{figure}

\textbf{Runtime Efficiency and Pareto Optimality.} We analyze the trade-off between predictive quality and computational cost in Figure~\ref{fig:main_pareto}, which plots the average rank of each method against its wall-clock runtime. FIRE is the fastest and dominates the Pareto frontier for both accuracy (NRMSE) and uncertainty quantification (NLL), residing in the optimal low-rank, low-latency region. 

\begin{figure}[ht!]
    \centering
    \includegraphics[width=1\linewidth]{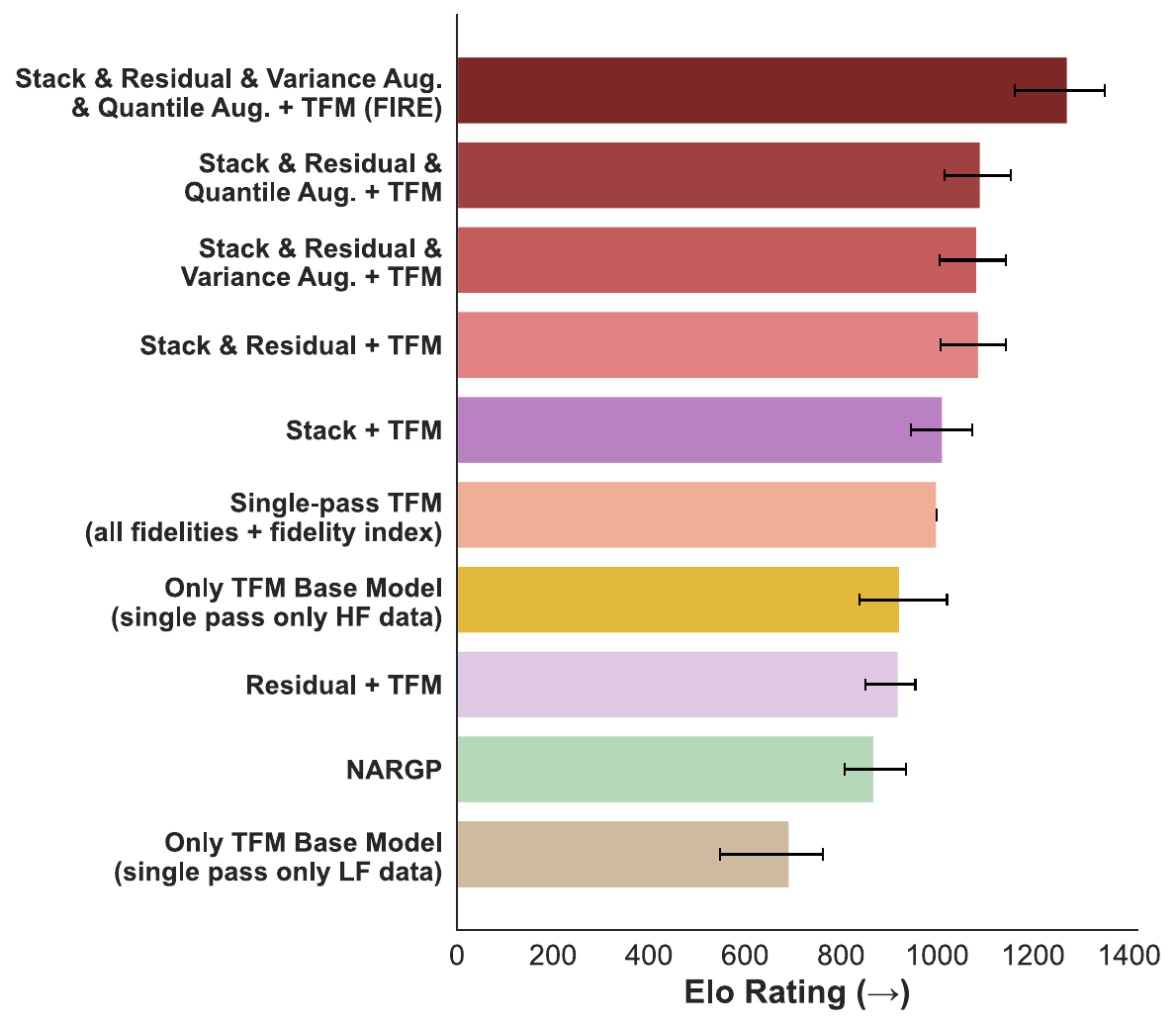}
    \caption{\textbf{Ablation study of FIRE components.} We compare the Elo rating of various configurations against NARGP, the strongest GP baseline. The full FIRE framework (Top), which combines residual learning with statistical augmentation, significantly outperforms partial implementations.}
    \label{fig:ablation_each_FIRE_component}
\end{figure}

\textbf{Robustness under Extreme Data Imbalance.} We next analyze performance as a function of the HF to LFsample ratio ($N_{\text{HF}}$/$N_{\text{LF}}$), a defining challenge in MF regression. As detailed in Figure \ref{fig:main_five_problems}, this stability holds across domains of problems. Figure~\ref{fig:main_splits} shows that FIRE maintains a clear advantage when HF data is extremely scarce (2–5\%), where GP-based and deep learning baselines degrade. The performance gap is the largest on real-world physics simulations and HPO tasks, confirming FIRE’s ability to effectively transfer information without overfitting.

 \begin{figure*}[ht]
    \includegraphics[width=.95\linewidth]{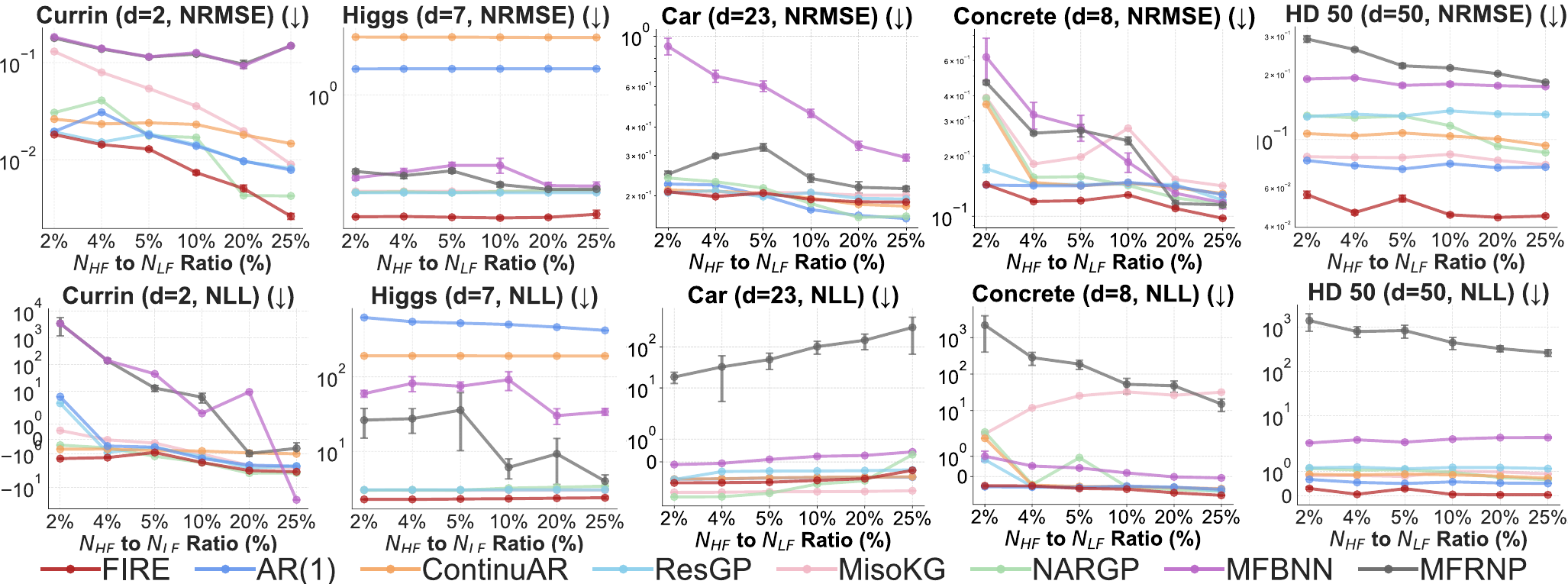}
    \caption{Performance (NRMSE and NLL) for representative synthetic (Currin, HD 50), physics-based (Car, Concrete), and HPO (Higgs) tasks. FIRE maintains top-ranked accuracy and uncertainty calibration in the extreme low-data regime (2–5\% HF ratio), where baseline performance degrades.}
    \label{fig:main_five_problems}
\end{figure*}

\begin{figure*}[ht]
    \centering
    \includegraphics[width=.9\linewidth]{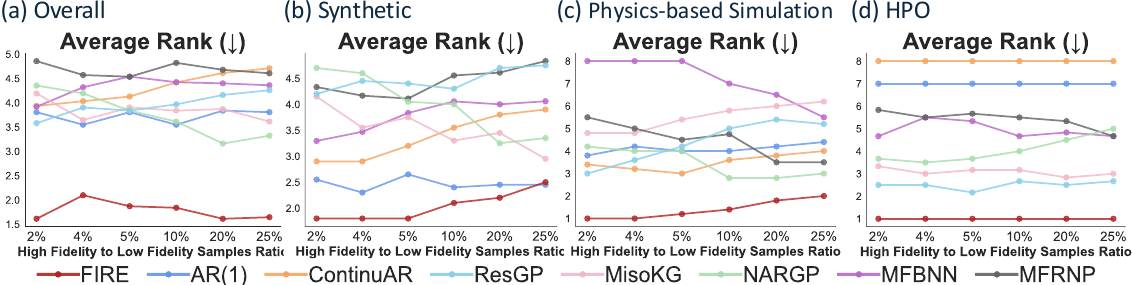}
    \caption{\textbf{Performance under varying data imbalance.} Aggregated results for problem domain as a function of the High-Fidelity to Low-Fidelity sample ratio ($N_{\text{HF}}/N_{\text{LF}}$ \%): (a)Overall (all 31 problems), (b)Synthetic, (c)Engineering, (d)HPO. FIRE demonstrates superior robustness compared to other baselines when HF data is extremely scarce ($2\%-5\%$), confirming the efficacy of TFM's zero-shot adaptation, especially for the real-world physics simulation and HPO tasks. }
    \label{fig:main_splits}
\end{figure*}

\section{Discussion} \label{Section:Discussion}

FIRE is faster and more accurate than current baselines, standing out when high-fidelity data is very scarce (2-5\%). This suggests that large-scale pre-training on diverse synthetic priors provides a strong inductive bias that compensates for HF data scarcity. In contrast, GP models struggle with kernel settings here, and deep models need extensive retraining. Our benchmarks also reveal a generalization gap in existing literature. While methods like ContinuAR and MFRNP excel in modeling spatially correlated PDE fields, they struggle to generalize to the heterogeneous feature spaces of tabular engineering and HPO problems. 
To disentangle the FIRE decomposition from the TFM backbone, we also instantiate FIRE with GP surrogates and observe consistent gains over standard MF GPs (Appendix \ref{Appendix:ablation:FIREGP}). We further ablate the role of in-context learning by swapping the base/residual learners, showing a clear drop when ICL is removed.
We also show that FIRE can be generalized to different TFMs (Appendix \ref{Appendix:Different_TFM_ablation}), and using post-hoc ensembling could help improve FIRE's performance (Appendix \ref{Appendix:post_hoc_ensemble}).

\section{Theoretical Analysis}\label{Section:Theoretical Analysis}

Section \ref{Section:Result} and  \ref{Section:Discussion} evaluate FIRE empirically. Here, we describe the posterior-approximation assumption underlying the use of TFMs and examine what Bayes-risk and distributional-representation results establish about LF predictive features. We distinguish these results from guarantees for finite fitted learners.

\subsection{TFMs as Bayesian Posterior Approximations}

We model the low-fidelity predictor as an amortized approximation to Bayesian posterior inference.
Specifically, we adopt an approximation-quality assumption from the theoretical and statistical foundations of TFMs  \citep{gordon2018metalearning,muller2021transformers,nagler2023statistical}, under which the predictive distribution produced by a TFM approximates the Bayesian posterior predictive distribution.
\begin{assumption}[TFM Posterior Approximation Quality] \label{Assumption:TFM Posterior Approximation Quality}
For the tasks considered, we assume that the LF TFM predictive law
$q_\theta(\cdot\mid x,\mathcal D_{\mathrm{LF}})$ adequately approximates
the posterior predictive law $\pi(\cdot\mid x,\mathcal D_{\mathrm{LF}})$
under its task prior. This is an approximation assumption, motivated by
the prior-data training objective discussed in Appendix~\ref{Assumption:TFM Posterior Approximation Quality}.
\end{assumption}
The pretraining objective minimizes a prior-averaged discrepancy. Yet it does not by itself guarantee accurate prediction for each new task or under distribution shift. We use the predictive mean, variance, and quantiles as features of the fitted predictive law. Accurate posterior moments require appropriate moment control in addition to distributional approximation.

\subsection{Justification for Distributional Conditioning}

We compare unrestricted Bayes risks under nested feature sets, and then state the implication for FIRE with a fixed LF context.

\begin{theorem}[Bayes risk under feature augmentation]\label{Theorem:Bayes Risk Monotonicity} %
Let $r=Y^{(T)}-\mu_\theta(X)$ satisfy $\mathbb E[r^2]<\infty$. Define
$Z_{\mathrm{mean}}=(X,\mu_\theta(X)),\;
 Z_{\mathrm{aug}}=(Z_{\mathrm{mean}},\sigma_\theta^2(X),
                  \{q_\theta^{(\tau)}(X)\}_{\tau\in\mathcal Q}),
$
and $\mathcal R(Z)=\mathbb E[(r-\mathbb E[r\mid Z])^2]$.
Then
{\footnotesize
$$
 \mathcal R(Z_{\mathrm{mean}})-\mathcal R(Z_{\mathrm{aug}})
 =\mathbb E\!\left[
 \bigl(\mathbb E[r\mid Z_{\mathrm{aug}}]
       -\mathbb E[r\mid Z_{\mathrm{mean}}]\bigr)^2\right]\geq 0.
$$} 
Equality holds if and only if the two conditional means agree almost surely.
\end{theorem}
\noindent\textit{Proof sketch.}
The conditional-expectation projection identity gives the displayed
formula \citep{casella2024statistical}; see Appendix~\ref{appendix:theory:Bayes_Risk_proof}.
For fixed LF context and inference configuration, including any randomness
used to construct the summaries, all added features are measurable
functions of $X$. Therefore both feature sets generate the same
sigma-algebra and their Bayes risks are equal. Explicit LF summaries may
nevertheless change representation and finite-sample performance for a
constrained learner. The theorem does not guarantee an improvement for
the fitted TFM, and feature augmentation can also increase its risk.

\subsection{Heteroscedastic Residuals and Variance Features}

Input-dependent noise has been studied extensively in heteroscedastic Gaussian process regression~\citep{kersting2007most,lazaro2011variational}. We consider a variance-coupling hypothesis that motivates supplying LF
predictive variance as an explicit residual-model feature.

\begin{assumption}[Variance-residual coupling]\label{assumption:Heteroscedastic Discrepancy} %
Let $r=Y^{(T)}-\mu_\theta(X)$ denote the residual used by FIRE.
With the LF context and inference configuration fixed, suppose
\[
 \operatorname{Var}(r\mid X)=g(\sigma_\theta^2(X))
\]
for a nonconstant, nondecreasing function
$g:[0,\infty)\to[0,\infty)$.
\end{assumption}

This hypothesis describes a possible relationship between LF predictive
uncertainty and residual variability; it is not a general property of MF
data. An explicit variance feature may help a constrained residual learner
adapt to such structure. However, at fixed LF context it is already a
function of $X$, so its omission does not imply misspecification of an
unrestricted predictor receiving $X$. Appendix~\ref{fireaudit:c5}
formalizes this distinction.

\begin{figure}[ht]
    \centering
    \includegraphics[width=1\linewidth]{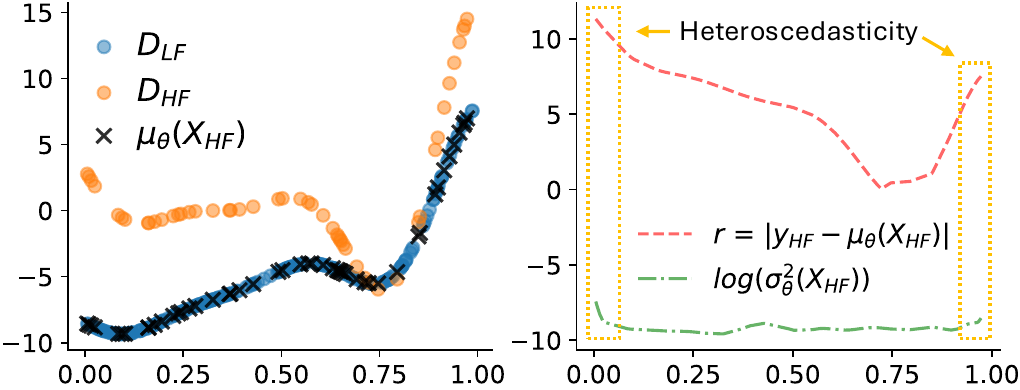}
    \caption{\textbf{Heteroscedastic cross-fidelity discrepancy on the Forrester($d=1$) benchmark} (right). Regions of larger LF uncertainty $\sigma_{\theta}^2$ coincide with larger and more variable HF residuals $r$, motivating uncertainty-conditioned residual modeling (left).}
    \label{fig:hetero_example_main}
\end{figure}

\subsection{Quantiles as Distributional Features}

Predictive quantiles summarize aspects of the LF predictive shape that are not generally determined by its mean and variance.

\begin{lemma}[Quantile representation]\label{lemma: Quantile Embedding}%
Let $F_\theta(\cdot\mid x,\mathcal D_{\mathrm{LF}})$ be the LF predictive
CDF and let $Q_\theta(u;x)$ be its generalized inverse for $0<u<1$.
The full quantile function uniquely determines this predictive law \citep{gushchin2017integrated}.
For $u_k=k/(K+1)$, define
\[
 \widehat F_K(y)=\frac1K\sum_{k=1}^K
       \mathbf 1\{Q_\theta(u_k;x)\leq y\}.
\]
Then $\sup_y|\widehat F_K(y)-F_\theta(y\mid x,\mathcal D_{\mathrm{LF}})|
\leq 1/(K+1)$; see Appendix~C.6.
\end{lemma}

This result concerns representation of the LF predictive law. It does not imply that the quantiles identify an arbitrary distribution or reduce HF prediction error.

\begin{assumption}[Finite second moment]\label{assumption:Non-Gaussian Residuals}
The residual $r$ satisfies $\mathbb E[r^2]<\infty$.
\end{assumption}

\begin{proposition}[Bayes risk under quantile augmentation]
Let $Z_{\mathrm{mv}}=(X,\mu_\theta(X),\sigma_\theta^2(X))$ and
$Z_{\mathrm{quant}}=(Z_{\mathrm{mv}},
\{q_\theta^{(\tau)}(X)\}_{\tau\in\mathcal Q})$.
Under the finite-second-moment assumption,
{\footnotesize$$
 \mathcal R(Z_{\mathrm{mv}})-\mathcal R(Z_{\mathrm{quant}})
 =\mathbb E\!\left[
 \bigl(\mathbb E[r\mid Z_{\mathrm{quant}}]
       -\mathbb E[r\mid Z_{\mathrm{mv}}]\bigr)^2\right]\geq0.
$$}
The inequality is strict exactly when the two conditional means differ
with positive probability. See Appendix~C.7.
\end{proposition}

Non-Gaussianity alone does not imply this strictness condition. For fixed LF context and inference configuration, the quantiles are functions of $X$, so the displayed risk difference is zero. Their value as explicit features for a constrained residual learner is assessed empirically \citep{Perdikaris2017,colombo2025warped}.

\section{Conclusion, Limitations, and Future work}

We presented FIRE, a training-free multi-fidelity regression framework that integrates tabular foundation models with residual and distribution-based uncertainty conditioning to tackle problems with very scarce high-fidelity data. Across 31 benchmark problems, including scalable synthetic functions, hyperparameter optimization data, and engineering simulations, FIRE consistently achieves SOTA performance while maintaining a favorable runtime profile, often outperforming extensive GP or deep learning model retraining from the existing methods. 

While FIRE eliminates the need for problem-specific training, it is limited by the TFM's context window, model size, and pre-training quality. Despite eliminating training overhead, FIRE's inference latency is bounded by the quadratic complexity of transformer attention~\citep {qu2025tabicl} too. Consequently, while significantly faster than deep learning baselines, its speed advantage over lightweight GP surrogates is more modest on larger contexts. Looking forward, we plan to expand FIRE to multi-fidelity optimization setting with its outstanding performance in predictive accuracy and uncertainty calibration.

\section*{Acknowledgements}

The authors thank Mohamed Elrefaie for providing the car design images of DrivAerNet datasets for Figure \ref{fig:main_figure}.

\section*{Impact Statement}
This work aims to accelerate scientific discovery and engineering design by improving the efficiency of multi-fidelity regression. By reducing the reliance on expensive high-fidelity simulations (e.g., in computational fluid dynamics or hyperparameter optimization), our method with faster runtime has the potential to lower the computational energy associated with data-intensive scientific modeling. We do not foresee any immediate negative ethical or societal consequences beyond the environmental impact of running inference on foundation models and standard considerations in automating engineering workflows.

\nocite{langley00}

\bibliography{MF_paper}
\bibliographystyle{icml2026}

\newpage
\appendix
\onecolumn

\startcontents[appendices]
\phantomsection

\section*{Table of Contents for Appendices}
\printcontents[appendices]{}{1}{}  %

\newpage

\section{Additional Experiment Results and Ablations}\label{Section:Ablation}

\subsection{Why Only Two Surrogates in FIRE? --- Ablation on the Number of Surrogates Used in FIRE}\label{appendix:surrogate_level}

A natural alternative to our bi-level strategy is the classical recursive formulation that explicitly models each step in the fidelity hierarchy, requiring $T-1$ distinct TFMs for a $T$-fidelity problem that commonly used by recursive GP co-kriging \citep{le2014recursive,cutajar2019deep}. We compare this variant against the standard FIRE implementation (2 surrogates) in Figure \ref{fig:numer_of_surrogates}. We observe that adding surrogates provides no distinct advantage in predictive accuracy (Elo), likely because the fidelity token sufficiently conditions the TFM on the source variance. However, the recursive strategy drastically degrades computational efficiency, increasing runtime by over $10\times$. Therefore, we adopt the bi-level approach to maximize computational throughput without compromising the fidelity of the posterior approximation.

\begin{figure}[ht]
    \centering
    \includegraphics[width=.55\linewidth]{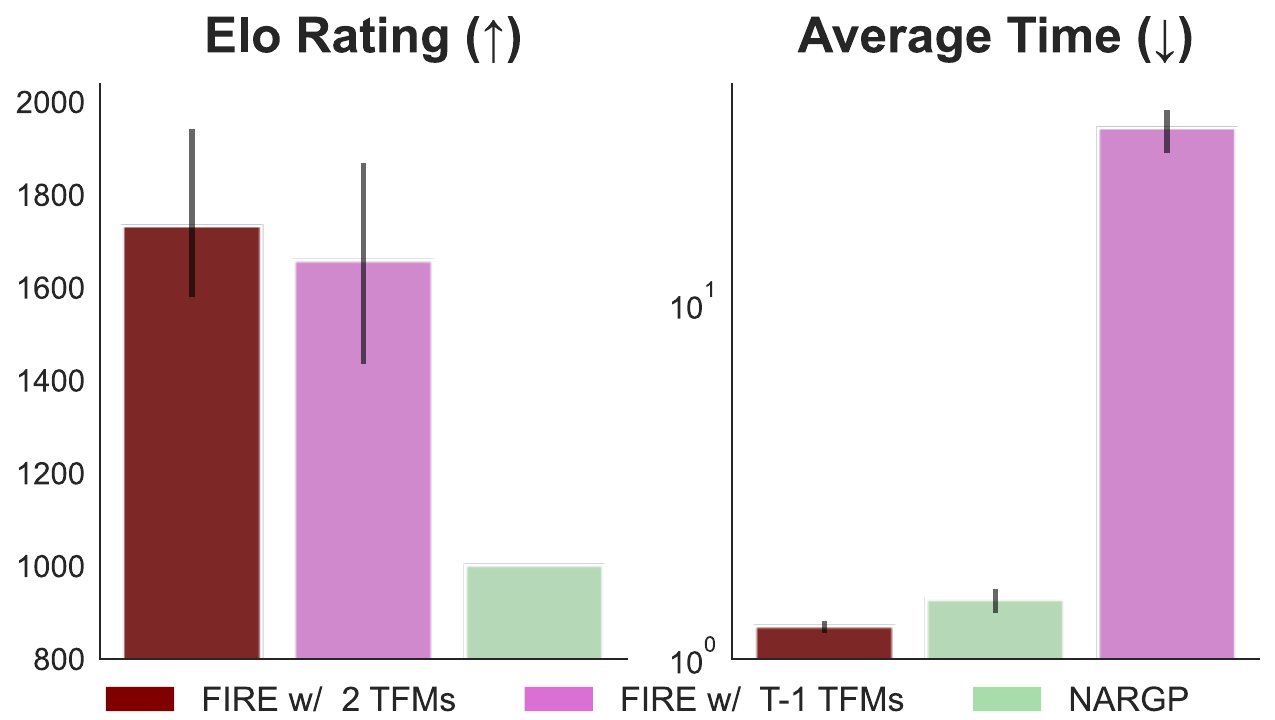}
    \caption{\textbf{Ablation on the number of surrogates. }We compare the standard bi-level FIRE (using 2 TFMs) against a fully recursive variant modeling every fidelity step (using $T-1$ TFMs) and the NARGP baseline. The bi-level approach achieves superior trade-offs between accuracy and efficiency, avoiding the prohibitive runtime costs of the recursive strategy.}
    \label{fig:numer_of_surrogates}
\end{figure}

\subsection{Can FIRE be Applied to GPs? --- Result of Applying FIRE Framework to GPs}\label{Appendix:ablation:FIREGP}

A key question is whether FIRE's performance stems solely from the Tabular Foundation Model (TFM) or from our distribution-conditioned residual framework. To test this, we instantiate FIRE(GP), replacing the TFM with a standard Gaussian Process to generate the $\mu$ and $\sigma^2$ statistics for the augmented context. As shown in Figure \ref{fig:FIREGP}, FIRE(GP) significantly outperforms traditional GP-based multi-fidelity methods (e.g., AR(1), NARGP, ResGP). These results provide empirical support for distributional features with the evaluated GP and TFM backbones. However, FIRE(TFM) still outperforms FIRE(GP), confirming that while the framework is effective, the TFM's pre-trained prior offers a distinct advantage over standard kernels in the low-data regime.

\begin{figure}[ht]
    \centering
    \includegraphics[width=.9\linewidth]{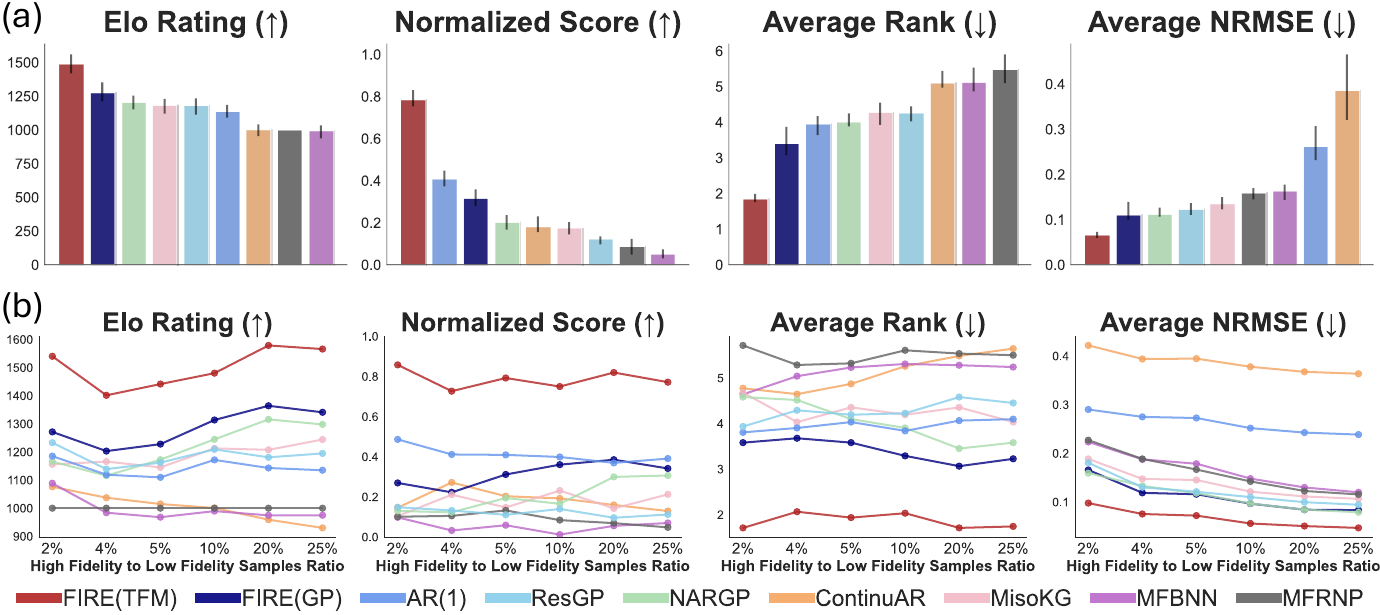}
    \caption{\textbf{Impact of surrogate choice on performance. }Replacing the TFM with a standard GP (FIRE(GP)) still yields state-of-the-art results compared to existing GP baselines, confirming that our residual learning framework drives significant gains. However, the performance gap between FIRE(TFM) and FIRE(GP) demonstrates the additional value provided by the TFM's in-context prior.}
    \label{fig:FIREGP}
\end{figure}

\subsection{Why Not Use Existing PFNs for Bayesian Optimization for Multi-fidelity Regression? --- Benchmarking against LC-PFN and FT-PFN}

We do not use FT-PFN \citep{rakotoarison2024ifbo} or LC-PFN \citep{lee2025costsensitive} in the main benchmark because they restrict input dimensions to fewer than 10. We validate this decision by comparing their performance on the compatible LCBench dataset in Figure \ref{fig:FIRE_vs_LC_FT}. We also include \textbf{Single-pass TabPFNv2.5 (all fidelities + fidelity index)} as a baseline to quantify the gain from solely the underlying model. This single-pass approach has all fidelity data as the input, and treats the fidelity index as an ordinary integer categorical feature within the joint dataset.

The results indicate that Single-pass TabPFNv2.5 (all fidelities + fidelity index) is already superior to both FT-PFN and LC-PFN. The specialized models struggle to match the generalization capabilities of the larger, general-purpose foundation model. Single-pass TabPFNv2.5 (all fidelities + fidelity index) achieves better accuracy and ranking without any task-specific adjustments. FIRE further extends this lead by incorporating distribution-conditioned residual learning. This demonstrates that a robust generalist model is preferable to restricted specialist models for these tasks.

\begin{figure}[ht]
    \centering
    \includegraphics[width=.95\linewidth]{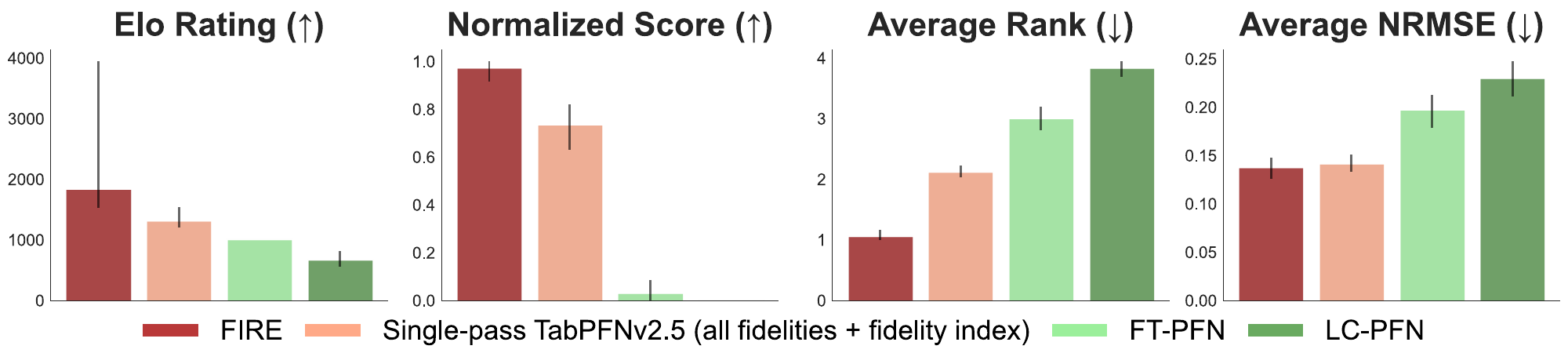}
    \caption{\textbf{Comparison against specialized HPO TFM surrogates.} We benchmark FIRE against state-of-the-art PFNs explicitly designed for learning curve extrapolation and Bayesian optimization (FT-PFN, LC-PFN) on the LCBench dataset. Despite FIRE using a general-purpose tabular foundation model, it significantly outperforms these specialized architectures in NRMSE and ranks among the top on the 6 HPO datasets in this study.}
    \label{fig:FIRE_vs_LC_FT}
\end{figure}

\subsection{What About Using Other TFMs for FIRE? --- Benchmarking FIRE with different TFMs}\label{Appendix:Different_TFM_ablation}

To assess whether FIRE's performance is driven by the specific choice of foundation model and the algorithm generalization, we instantiate the framework with four different TFMs: RealTabPFN \citep{garg2025real}, TabPFNv2.5, TabPFNv2 \citep{hollmann2025accurate}, and Mitra \citep{zhang2025mitra}. As shown in Figure \ref{fig:Different_TFM}, all TFM-based FIRE variants significantly outperform the strongest GP baseline, NARGP, confirming the robustness of the distribution-conditioned residual framework. Among the TFMs, TabPFNv2.5 and RealTabPFN achieve the highest Elo ratings and lowest NRMSE, with RealTabPFN showing a slight edge likely due to its fine-tuning on real-world data distributions. While Mitra outperforms the GP baseline, it lags behind the TabPFN family. We attribute this to the fact that Mitra’s current AutoGluon interface lacks native support for direct quantile and variance outputs, necessitating an approximation via a split conformal wrapper that may degrade the quality of the distributional features used for residual conditioning.

\begin{figure}[ht]
    \centering
    \includegraphics[width=.95\linewidth]{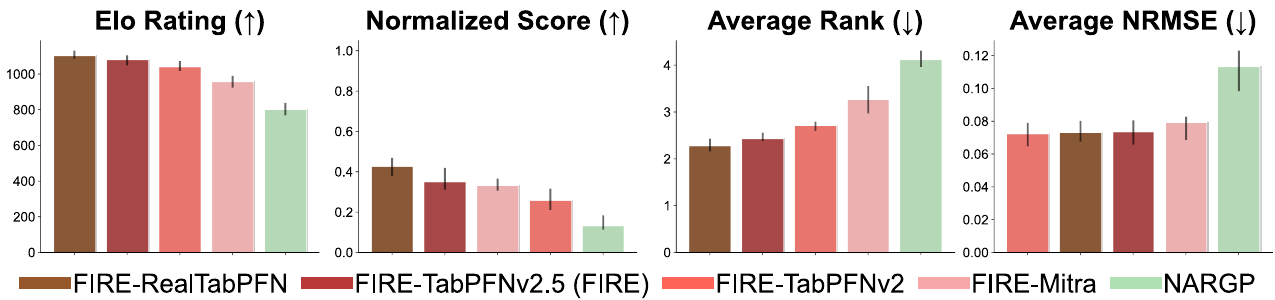}
    \caption{\textbf{Benchmark of FIRE instantiated with different TFM.} We compare FIRE using RealTabPFN, TabPFNv2.5, TabPFNv2, and Mitra against the strongest GP-baseline (NARGP). While all TFM-based variants outperform the GP baseline, TabPFN-based backbones yield the strongest results.}
    \label{fig:Different_TFM}
\end{figure}

\subsection{Does FIRE Maintain SOTA on Nested Data? --- Performance Comparison on Nested vs. Non-nested Inputs}

All benchmarks in the main text use disjoint (non-nested) input sets to mimic realistic, asynchronous data collection. We investigate whether this protocol unfairly penalizes autoregressive baselines that typically favor nested designs. We compare the performance of FIRE against NARGP, AR(1), and ResGP on both nested ($X_{HF} \subseteq X_{LF}$) and disjoint ($X_{HF} \not\subseteq X_{LF}$) inputs of our synthetic benchmarks, since the real-world simulation and HPO datasets come disjointly. Figure \ref{fig:Disjoint_grid} shows the results across all benchmarks.

We observe two key trends. First, GP-based autoregressive methods (NARGP, AR(1), ResGP) consistently perform better in the nested setting, confirming their sensitivity to data alignment. However, even under these ideal conditions, they still lag behind FIRE in both Elo rating and average rank. Second, the performance gap widens significantly in the disjoint regime. While the baselines degrade when inputs are unpaired, FIRE maintains superior performance and stability. This demonstrates that FIRE remains SOTA in the nested setting while offering unique robustness to the irregular, non-nested data structures common in real-world applications.

\begin{figure}[htbp!]
    \centering
    \includegraphics[width=.72\linewidth]{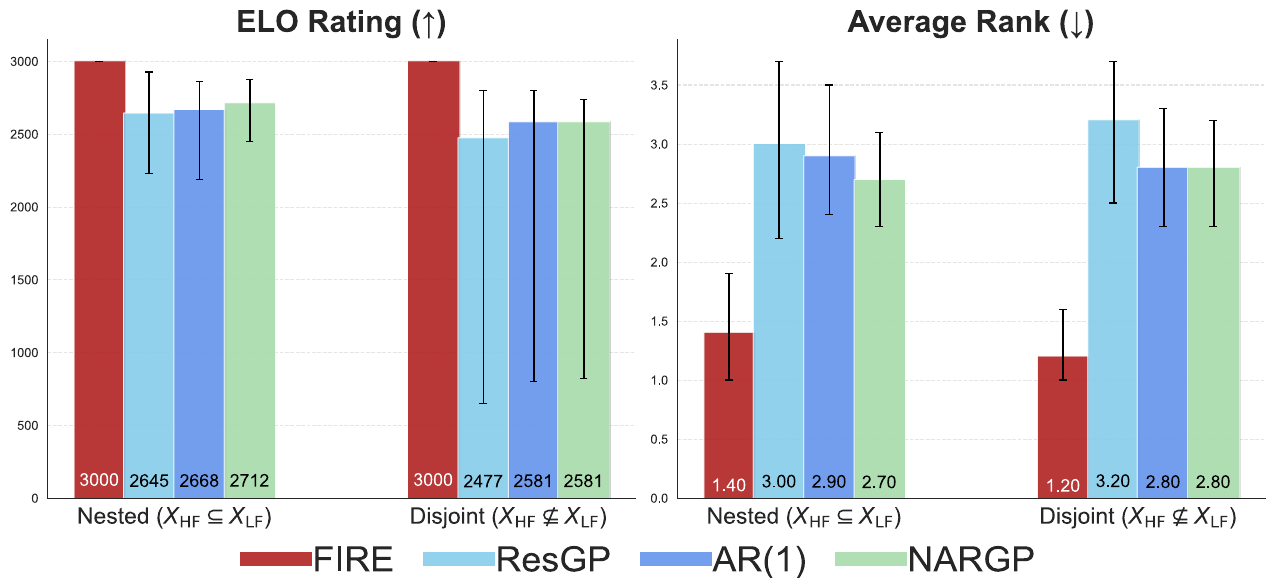}
    \caption{\textbf{Robustness to non-nested data. }We compare FIRE against autoregressive GP baselines (NARGP, AR(1), ResGP) on nested ($X_{HF} \subseteq X_{LF}$) versus disjoint ($X_{HF} \not\subseteq X_{LF}$) input designs. Elo ratings (left) are calibrated to FIRE=3000. While baselines benefit from nesting, FIRE achieves the highest Elo and lowest rank in both settings, demonstrating superior generalization to disjoint data.}
    \label{fig:Disjoint_grid}
\end{figure}

\subsection{Is In-Context Learning (ICL) Necessary? --- Ablation of Base and Residual Learner Models}\label{Appendix:ICL_ablation}

We investigate if FIRE's performance comes from the distributional features or the TFM architecture. We test four combinations of surrogates. We use either a GP or a TFM for the low-fidelity (base) model and the high-fidelity (residual) model. Figure \ref{fig:res_model} shows the results. We see two trends: First, configurations using a TFM for the residual model consistently outperform those using a GP. This suggests that the TFM's flexible prior handles the complex, non-Gaussian structure of residuals better than a GP, even when both receive the same distributional summaries. Second, using a TFM for the base model (hatched bars) further improves accuracy compared to a GP base model.

\begin{figure}[htbp!]
    \centering
    \includegraphics[width=.72\linewidth]{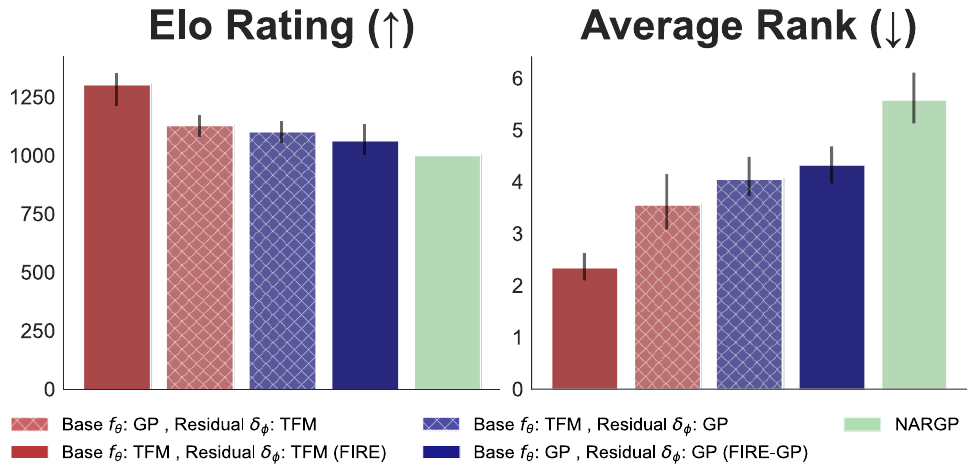}
    \caption{\textbf{Impact of surrogate choice for base $f_{\theta}$ and residual $\delta_{\phi}$ modeling.} We compare combinations of GP and TFM backbones on their accuracy (NRMSE). Red bars indicate a TFM residual model, while navy bars indicate a GP. Hatching indicates a mixture of models in the FIRE framework. The TFM's ability to model residuals via ICL provides a distinct advantage over standard kernels.}
    \label{fig:res_model}
\end{figure}

\subsection{Can FIRE Be Improved Further? – Enhancing Residual Modeling via Post-Hoc Ensembling} \label{Appendix:post_hoc_ensemble}

We further investigate if improving the residual model directly translates to better multi-fidelity performance. We apply post-hoc ensembling (PHE) to the residual TFM, a technique introduced in \citet{hollmann2025accurate} to boost the residual model accuracy. Figure \ref{fig:post_hoc_ensemble} shows that FIRE with PHE outperforms the standard FIRE implementation. This confirms that the quality of the residual prediction is a bottleneck for multi-fidelity regression. However, the improvement comes with a computational cost. The inference time increases by an order of magnitude.

Finally, we tested whether training the residual from scratch, rather than relying on the TFM's generic prior, yields a better accuracy perofrmance. Holding the TabPFNv2.5 base fixed, we swap the residual learner for scratch-trained GP, Random Forest, XGBoost, and LightGBM models. Figure \ref{fig:res_model_all} shows the TFM residual wins, outperforming GP (1000 vs. 774 Elo), RF (738), and XGBoost (605). This indicates the pre-trained TFM prior, not task-specific fitting, drives FIRE's gains, consistent with TabPFNv2.5 leading tree and gradient boosting methods on TabArena \cite{erickson2025tabarena}.

\begin{figure}[ht]
    \centering
    \includegraphics[width=.55\linewidth]{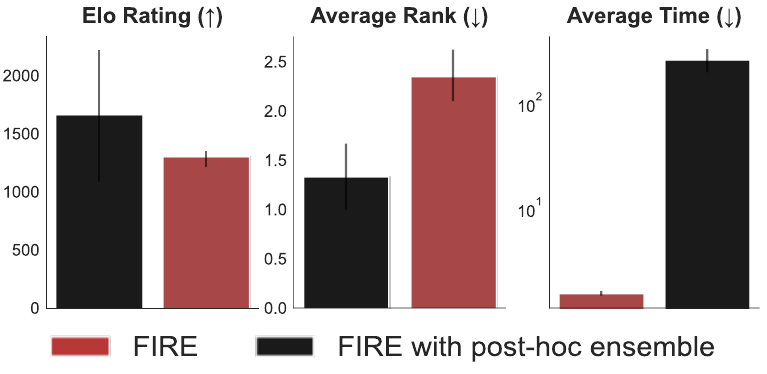}
    \caption{\textbf{Impact of post-hoc ensembling on FIRE. }Ensembling the residual model improves predictive accuracy (Elo and rank of NRMSE) but increases inference time significantly.}
    \label{fig:post_hoc_ensemble}
\end{figure}

\begin{figure}[htbp!]
    \centering
    \includegraphics[width=.85\linewidth]{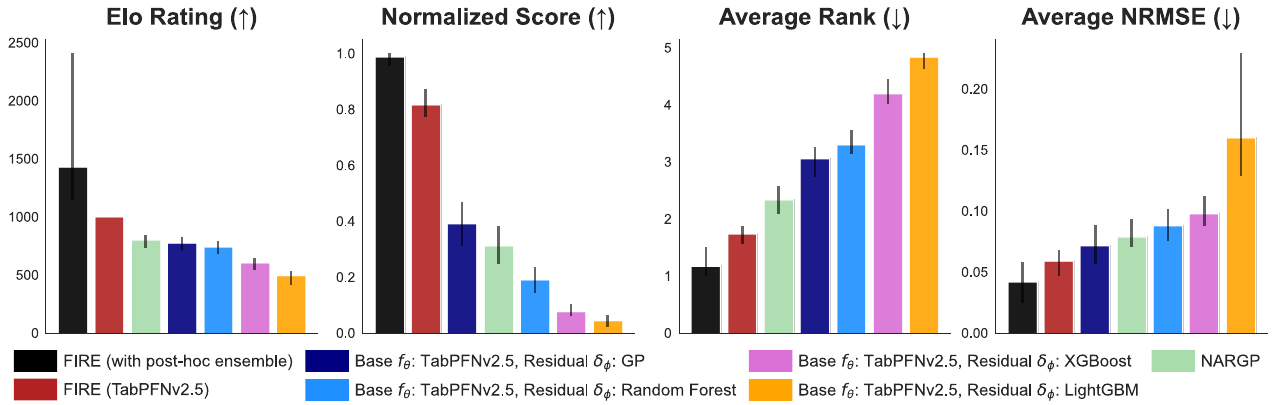}
    \caption{\textbf{Residual model choice for $\delta_\phi$.} With the TabPFNv2.5 base fixed, we replace the TFM residual with scratch-trained GP, RF, XGBoost, and LightGBM models (Elo calibrated to FIRE=1000). The TFM residual achieves the best accuracy, confirming that training the residual from scratch does not improve FIRE's Pareto frontier.}
    \label{fig:res_model_all}
\end{figure}

\newpage

\subsection{Is the uncertainty truly calibrated? – Uncertainty Quantification (UQ) Analysis with Reliability Diagram and Interval Scores} \label{Appendix:UQ_daigram_IS}

To verify that FIRE's uncertainty estimates are genuinely calibrated rather than merely useful as features, we evaluate two complementary diagnostics across all benchmarks. The reliability diagram compares nominal coverage to empirical coverage: for each quantile level $p$, we form the predicted Gaussian quantile $\hat{q}_p = \mu + \sigma\,\Phi^{-1}(p)$ from each model's predictive mean and variance, and measure the empirical frequency of test points falling below it. The perfect calibration lies on the diagonal. The Interval Score (IS), computed with the same formulation as \citet{Yousefpour2024GPPlus}, evaluates the central prediction interval by rewarding narrow intervals and penalizing observations that fall outside, with lower scores indicating better-calibrated intervals. As shown in Figure \ref{fig:UQ_all}, both FIRE variants track the calibration diagonal closely and rank among the top methods on IS. FIRE-GP results in the strongest pure calibration, while FIRE-TabPFNv2.5 remains competitive but slightly less calibrated, consistent with the limitations of current TFM predictive distributions \cite{nagler2023statistical,leeuwen2025conformal}. This supports recommending FIRE-GP for calibration-critical applications and FIRE-TabPFNv2.5 for accuracy- and runtime-driven ones.

\begin{figure}[htbp!]
    \centering
    \includegraphics[width=1\linewidth]{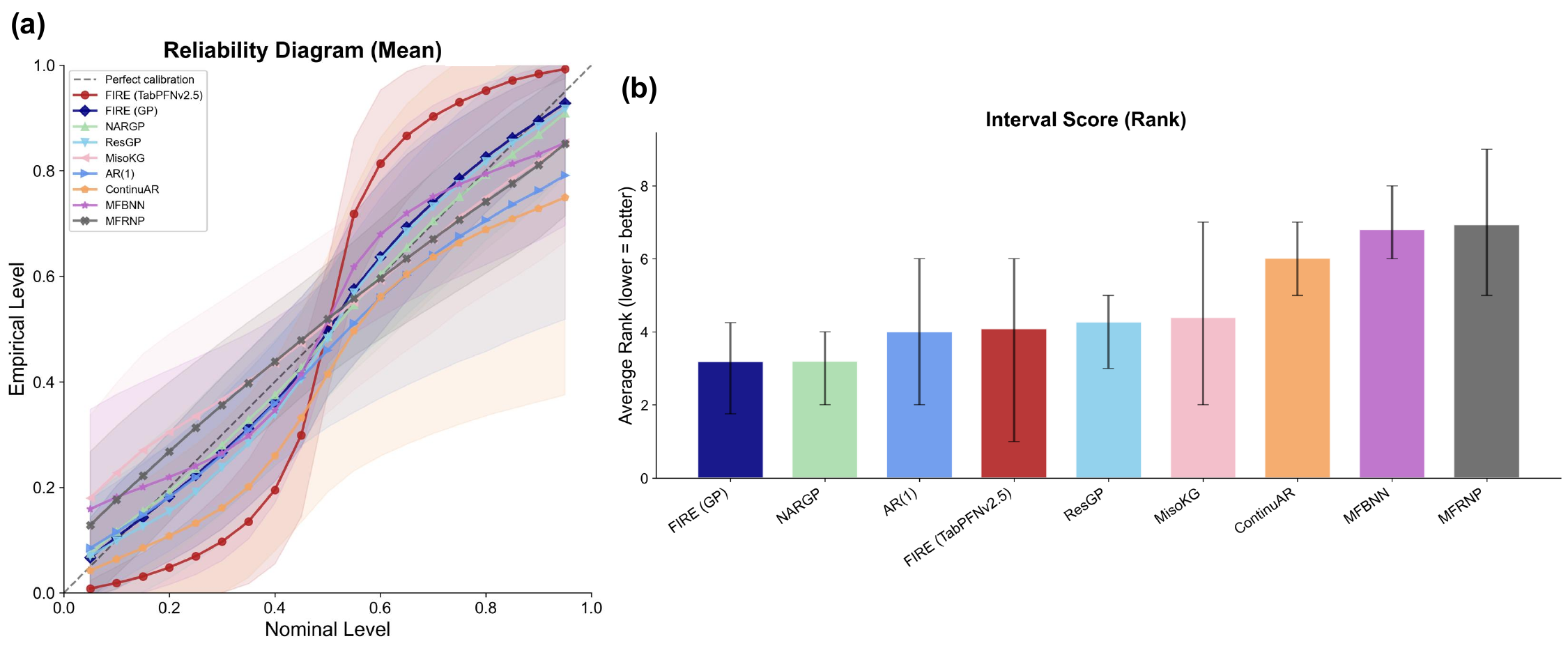}
    \caption{\textbf{Uncertainty quantification analysis.} (a) Reliability diagram (empirical vs. nominal coverage; the dashed diagonal denotes perfect calibration) and (b) average rank of Interval Score (lower is better) across all benchmarks. Both FIRE variants are well-calibrated and top-ranked on IS, with FIRE-GP achieving the best calibration and FIRE-TabPFNv2.5 remaining competitive.}
    \label{fig:UQ_all}
\end{figure}

\newpage

\section{Full Statistical Analysis and Extended Results}\label{Appendix:full_results}

\subsection{Evaluation Metrics and Aggregation}\label{Appendix:B2:metrics}

We expand upon the experimental protocol in Section \ref{Sectoin:Experiments:Metrics} by including $R^2$ assessment and detailing our aggregation hierarchy. The individual metric formulations are:
$$\text{NRMSE} = \frac{\sqrt{\frac{1}{N} \sum_{i=1}^{N} (y_i - \hat{y}_i)^2}}{y_{\max} - y_{\min}}$$

$$\text{NLL} = \frac{1}{2N} \sum_{i=1}^{N} \left( \ln(2\pi \sigma_i^2) + \frac{(y_i - \hat{y}_i)^2}{\sigma_i^2} \right)$$

$$R^2 = 1 - \frac{\sum_{i=1}^{N} (y_i - \hat{y}_i)^2}{\sum_{i=1}^{N} (y_i - \bar{y})^2}$$

To ensure fair comparison across diverse problem scales, we utilize four aggregation metrics following \citet{erickson2025tabarena}: 
\begin{enumerate}
    \item \textbf{Elo Rating}, a rating system for pairwise comparisons, in which a model’s score reflects its predicted probability of winning against other models. A 400-point Elo difference corresponds to a 10-to-1 (91\%) expected win rate. We calibrate MFRNP to 1000 as the baseline.
    \item \textbf{Normalized Score}, where the best and median models are scaled to 1.0 and 0.0, respectively, neutralizing scale differences between datasets.
    \item \textbf{Average Rank}, providing a robust, scale-invariant ordering.
    \item \textbf{Statistical Ranking}, visualized via \textbf{Critical Difference (CD) diagrams} based on the Friedman test and Nemenyi post-hoc analysis ($\alpha=0.05$).
\end{enumerate}

\begin{figure}[htb!]
    \centering
    \includegraphics[width=.9\linewidth]{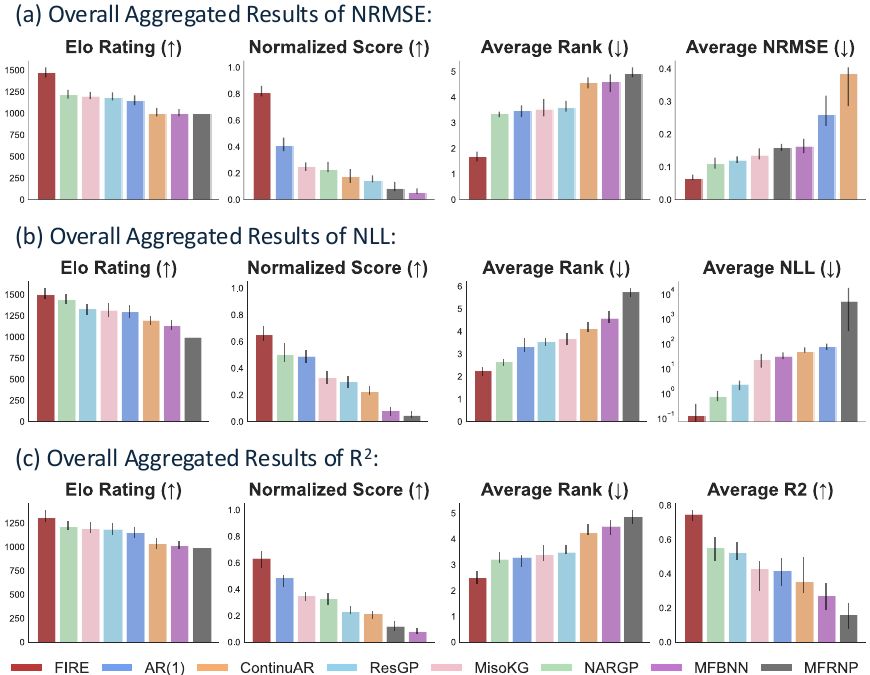}
    \caption{\textbf{Overall aggregate performance across 31 problems.} Performance is aggregated via Elo, Normalized Score, Average Rank, and raw averages for NRMSE, NLL, and $R^2$. FIRE dominates all aggregation schemes, confirming its status as the state-of-the-art method for accuracy and uncertainty quantification.}
    \label{fig:Full_bar}
\end{figure}

\subsection{Overall Performance Analysis}

\paragraph{Global Statistical Ranking. }Figure \ref{fig:Full_bar} reports the combined performance over 31 benchmark problems under the four aggregation protocols in \ref{Appendix:B2:metrics}. FIRE reaches the top Elo rating and Normalized Score on all three error metrics (NRMSE, NLL, and $R^2$). The performance advantage of FIRE is consistent across aggregation schemes and not metric-dependent. Notably, in terms of Average Rank, FIRE maintains a significant lead ($\approx 1.8$) over the strongest GP baseline (NARGP, rank $\approx 3.5$), validating the effectiveness of the TFM prior in diverse tabular settings.

\paragraph{Robustness to Data Imbalance. }Figure \ref{fig:Full_splits} decomposes performance by the high-fidelity to low-fidelity sample ratio ($N_{HF}/N_{LF}$). While strong baselines like NARGP and MFKG degrade rapidly as $N_{HF}$ drops below 5\%, FIRE maintains superior stability in extremely data-scarce regimes ($2-5\%$). This confirms that the frozen prior of the TFM acts as a robust regularizer, preventing the overfitting common in kernel-based methods when high-fidelity data is sparse.

\paragraph{Statistical Significance and Win Rates.} To rigorously verify these rankings, Figure \ref{fig:CD_All} presents Critical Difference (CD) diagrams based on the Nemenyi post-hoc test ($\alpha=0.05$). FIRE is statistically distinguishable from all GP-based baselines in NRMSE, forming a distinct top-tier significance group. This dominance is further corroborated by the pairwise win-rate matrix in Figure \ref{fig:win_rate_matrix}, where FIRE achieves a positive win rate ($>60\%$) against every competing baseline, including deep learning methods like MFRNP and MFBNN.

\paragraph{Runtime Analysis.} Efficiency and Scaling. Figure 11 analyzes algorithmic scalability on the synthetic HD suite ($d=10$ to $50$). Results highlight distinct efficiency tiers: neural methods (MFBNN, MFRNP) are computationally inefficient due to training convergence requirements, while standard GPs form a middle tier bounded by kernel fitting costs. FIRE demonstrates superior efficiency, consistently achieving the lowest runtime due to its zero-shot nature. Although limited by the quadratic attention mechanism of the underlying transformer~\citep{qu2025tabicl}, FIRE's zero-shot paradigm still offers a decisive speed advantage for high-dimensional regression tasks.

\begin{figure}[htbp!]
    \centering
    \includegraphics[width=1\linewidth]{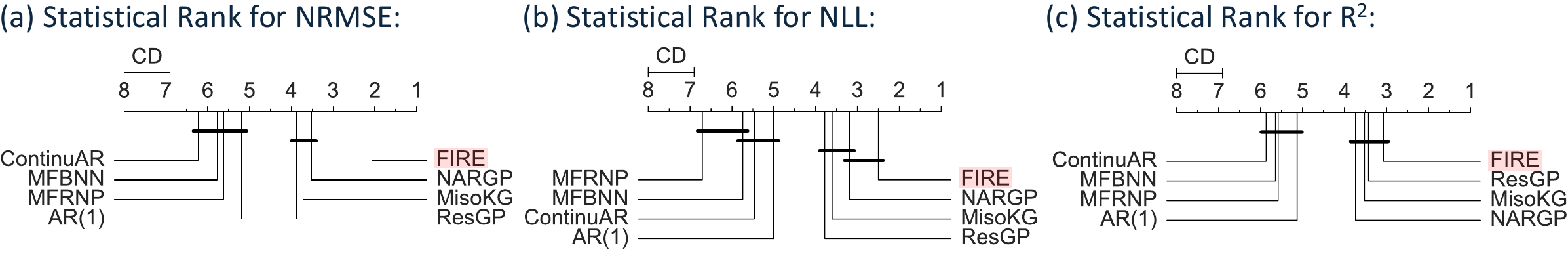}
    \caption{\textbf{Critical Difference (CD) diagrams} based on the Nemenyi post-hoc test ($\alpha=0.05$) for (a) NRMSE, (b) NLL, and (c) $R^2$. Methods connected by a horizontal bar are not statistically significantly different. FIRE ranks first statistically across all three metrics, specifically forming a distinct significance group from the GP baselines in accuracy (NRMSE).}
    \label{fig:CD_All}
\end{figure}

\begin{figure}[htbp!]
    \centering
    \includegraphics[width=.95\linewidth]{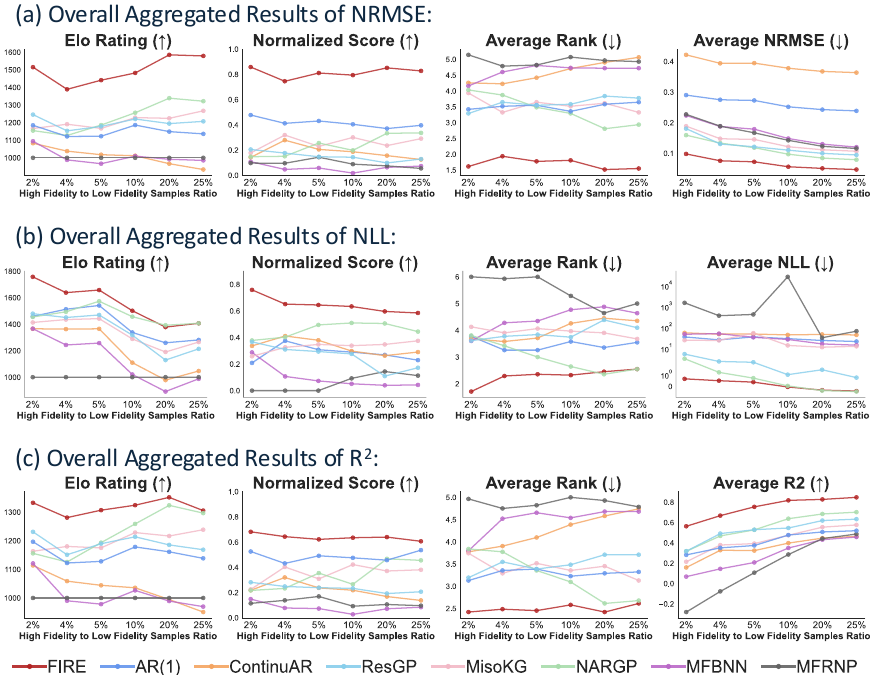}
    \caption{\textbf{Performance vs. Data Imbalance. }Aggregate metrics (Elo, Normalized Score, Rank, Raw) plotted against the HF/LF sample ratio. FIRE exhibits high stability in low-data regimes (2-5\%), significantly outperforming baselines which require larger HF budgets to converge.}
    \label{fig:Full_splits}
\end{figure}

\begin{figure}[htbp!]
    \centering
    \includegraphics[width=.6\linewidth]{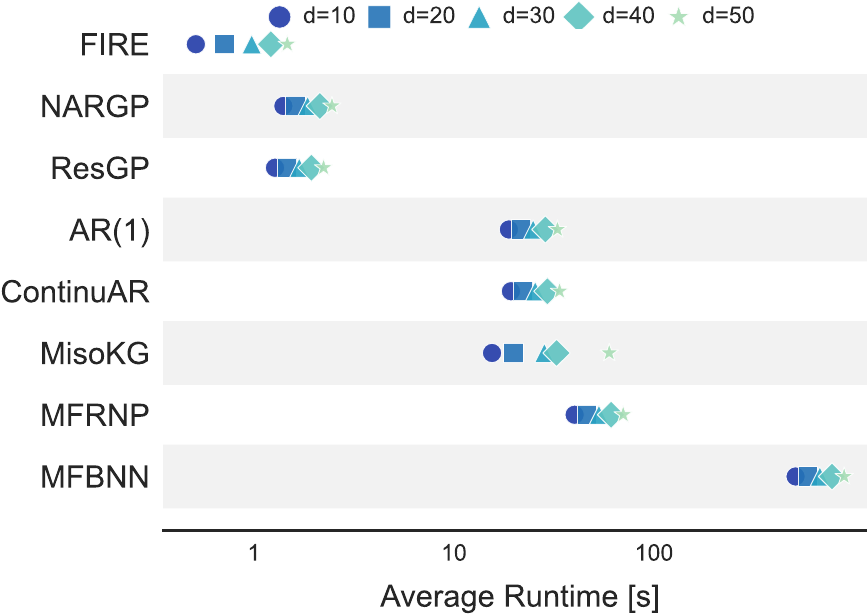}
    \caption{\textbf{Runtime scaling with input dimensionality.} This figure reports wall-clock runtime on the HD benchmark suite as the input dimension $d$ increases from 10 to 50. FIRE has the lowest runtime at all dimensions. GP and deep learning baselines take longer because they rely on kernel fitting or extensive training.}
    \label{fig:runtime_scaling}
\end{figure}

\begin{figure}[htbp!]
    \centering
    \includegraphics[width=.6\linewidth]{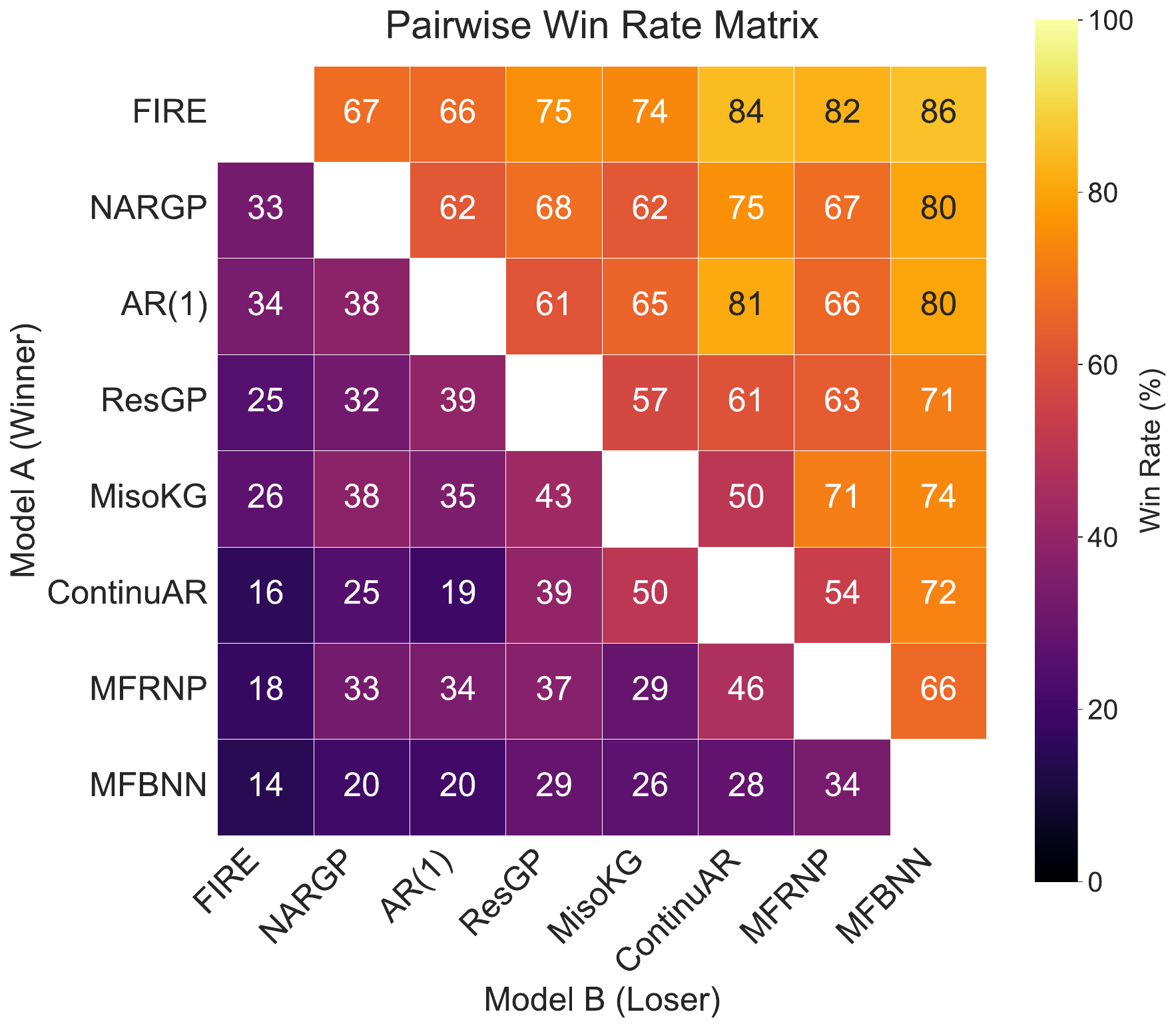}
    \caption{\textbf{Pairwise win-rate comparison for all problems. }The heatmap displays the percentage of tasks where the algorithm on the Y-axis outperforms the algorithm on the X-axis. FIRE consistently beats all baselines across the majority of tasks (higher numbers are better).}
    \label{fig:win_rate_matrix}
\end{figure}

\subsection{Performance by Problem Category and Data Imbalance}
Figure 12-14 decompose performance across Synthetic, Physics-based, and HPO domains under varying high-fidelity data budgets ($2-25\%$):
\begin{itemize}

    \item \textbf{Synthetic Functions:} FIRE exhibits superior robustness in the low-data regime ($<5\%$), outperforming deep learning baselines in NLL and $R^2$ while remaining competitive with GP methods (AR(1), NARGP, MFKG) on smooth surfaces.

    \item \textbf{Physics-based Simulations:} FIRE dominates engineering benchmarks, achieving the lowest NRMSE and NLL across all ratios and outperforming strong baselines such as ContinuAR, which is designed for solving engineering multi-fidelity problems (e.g., PDE simulations). This confirms that FIRE's distributional conditioning effectively captures the heteroscedastic residuals inherent in complex physical solvers (e.g., CFD), where standard kernels can still suffer from overfitting.

    \item \textbf{HPO Tasks:} On LCBench problems, FIRE consistently ranks the best in all metrics, validating its dominant performance as a surrogate for irregular hyperparameter landscapes. Its stability in the extreme sparsity regime ($2\%$ HF data) highlights its potential for efficient, few-shot multi-fidelity optimization.

\end{itemize}

\begin{figure}[htbp!]
    \centering
    \includegraphics[width=.85\linewidth]{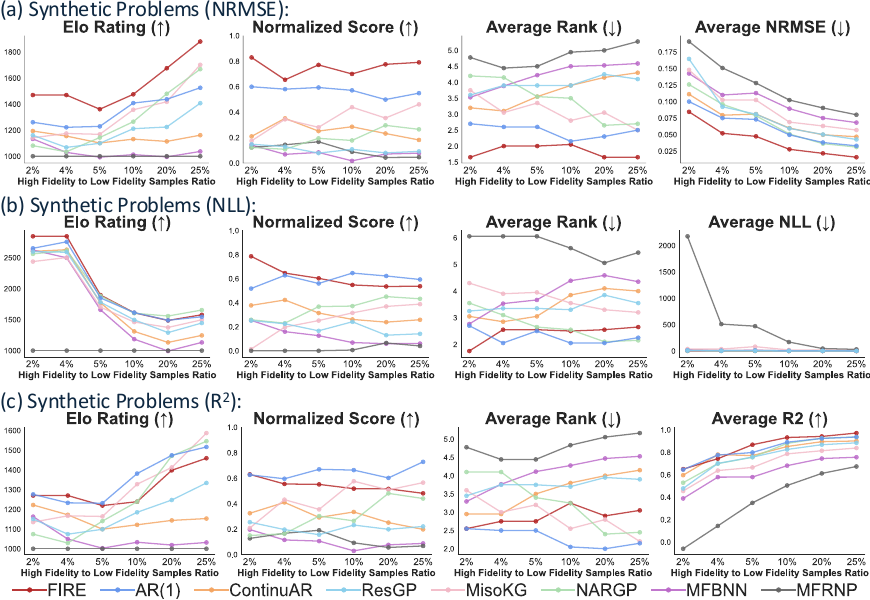}
    \caption{\textbf{Synthetic Benchmarks: Performance vs. data imbalance.} Aggregated metrics vs. HF/LF ratio for 18 synthetic tasks. While GP-based methods like AR(1) remain competitive, FIRE maintains top-tier performance, particularly in the low-data regime ($2-5\%$) for NLL and $R^2$.}
    \label{fig:Synthetics_splits}
\end{figure}

\begin{figure}[htbp!]
    \centering
    \includegraphics[width=.85\linewidth]{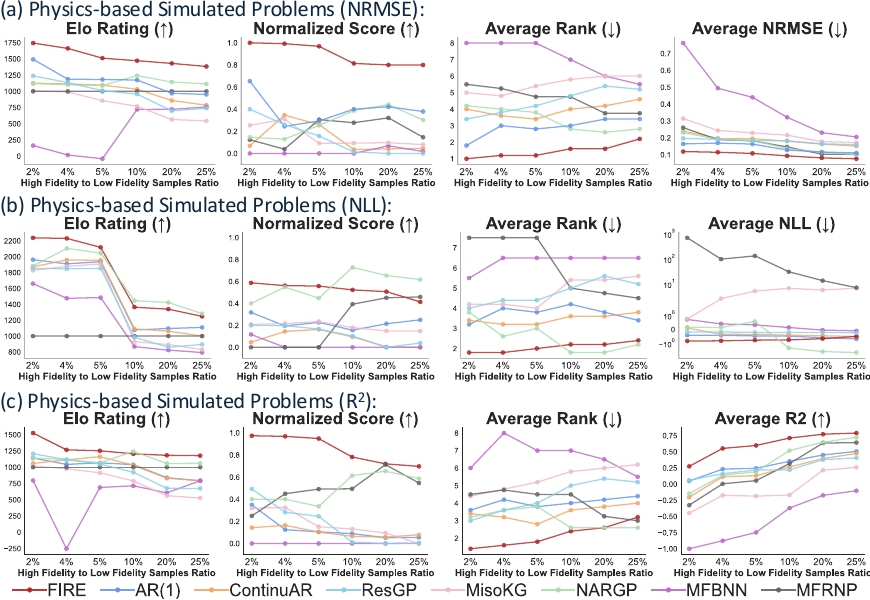}
    \caption{\textbf{Physics-based Simulations Benchmarks: Performance vs. data imbalance.} Aggregated results for the 7 engineering and physics problems. FIRE demonstrates a clear advantage here, significantly outperforming all baselines in NRMSE and NLL.}
    \label{fig:Engineering_splits}
\end{figure}

\begin{figure}[htbp!]
    \centering
    \includegraphics[width=.85\linewidth]{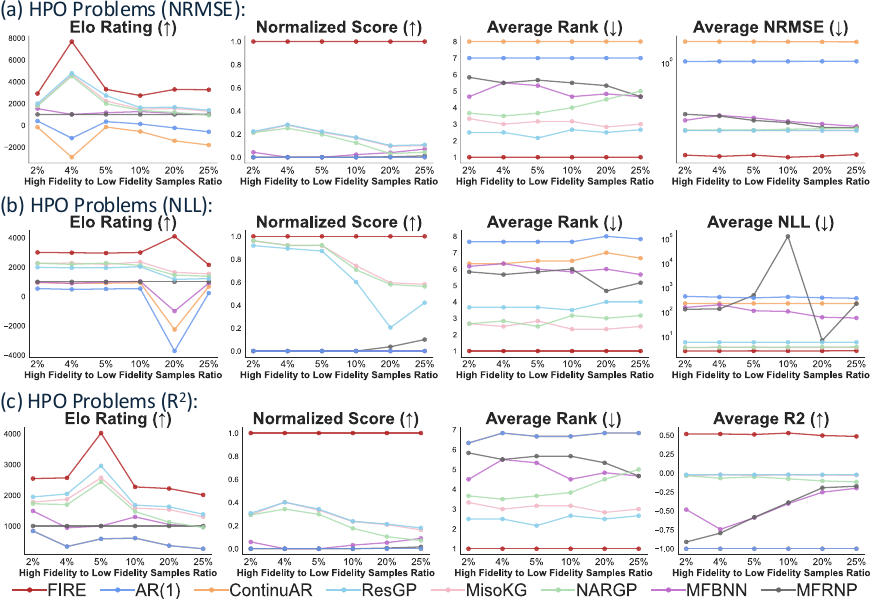}
    \caption{\textbf{Hyperparameter Optimization (HPO) tasks: Performance vs. data imbalance.} Aggregated results for the 6 LCBench learning curve datasets. FIRE outperforms all baselines, validating its effectiveness in modeling the irregular response surfaces typical of hyperparameter optimization landscapes.}
    \label{fig:HPO_splits}
\end{figure}

\FloatBarrier
\newpage

\section{Theoretical Analysis}\label{Section:Theory}
In this section, we establish the theoretical foundations for FIRE. Our analysis builds upon the statistical framework of Prior-Data Fitted Networks (PFNs)~\citep{muller2021transformers,nagler2023statistical}, while accounting for the unique properties of TabPFN as a surrogate model and the principles of generalized Bayesian stacking~\citep{yao2018using}.

\subsection{Prior-Data Training and Posterior Approximation}

The prior-data objective motivates a Bayesian interpretation of Prior-Data Fitted Networks (PFNs) predictions~\citep{muller2021transformers,nagler2023statistical}. We separate its population optimum from the approximation quality of the finite pretrained model used in FIRE.

\begin{theorem}[Population prior-data objective]
Let $(D,X,Y)$ have the joint law induced by a task prior, with posterior
predictive density $\pi(y\mid X,D)$. Suppose the expectations below are
finite. For a candidate conditional density $q$, define
$\mathcal L(q)=\mathbb E[-\log q(Y\mid X,D)]$. Then
\[
 \mathcal L(q)-\mathcal L(\pi)
 =\mathbb E_{D,X}\!\left[
   \operatorname{KL}\bigl(\pi(\cdot\mid X,D)
                          \Vert q(\cdot\mid X,D)\bigr)\right].
\]
Consequently, a global population minimizer within a model class is
KL-optimal within that class. If the class contains $\pi$, an attained
population minimum equals $\pi$ almost surely under the prior-data law.
\end{theorem}
\begin{proof}
Condition on $(D,X)$ and subtract the conditional cross-entropy evaluated
at $\pi$. The resulting integral is the conditional KL divergence.
Taking expectations proves the identity.
\end{proof}
This population statement does not ensure that finite pretraining reaches
the optimum or approximates every new task accurately. Posterior
consistency also requires a specified sampling model and additional
conditions on the prior and model class \citep{nagler2023statistical}, since prior support alone is not a
finite-sample guarantee. We do not claim that those conditions have been
verified for FIRE's multi-fidelity contexts.

For FIRE, adequate predictive approximation is an assumption, rather than
a consequence established for each benchmark. The mean, variance, and
quantiles are summaries of the fitted predictive law. Closeness of
distributions alone does not guarantee accurate unbounded moments without
additional moment control. We therefore use these summaries as features
and evaluate their predictive utility empirically.

\subsection{Proof of Theorem \ref{Theorem:Bayes Risk Monotonicity}}\label{appendix:theory:Bayes_Risk_proof}
We use the standard conditional-expectation projection identity \citep{casella2024statistical}. 

\begin{proof}
Let $\mathcal G_0=\sigma(Z_{\mathrm{mean}})$ and
$\mathcal G_1=\sigma(Z_{\mathrm{aug}})$, so
$\mathcal G_0\subseteq\mathcal G_1$. Suppose $r\in L^2$ and set
$m_j=\mathbb E[r\mid\mathcal G_j]$. Expanding
$r-m_0=(r-m_1)+(m_1-m_0)$ yields
\[
 \mathbb E[(r-m_0)^2]
 =\mathbb E[(r-m_1)^2]+\mathbb E[(m_1-m_0)^2],
\]
because
$\mathbb E[(r-m_1)(m_1-m_0)]
 =\mathbb E[(m_1-m_0)\mathbb E[r-m_1\mid\mathcal G_1]]=0$.
This proves the identity, non-negativity, and the conditional-mean equality
criterion. With LF context and inference configuration fixed, both
sigma-algebras equal $\sigma(X)$, so the risk difference is zero.
\end{proof}
Conditional independence of the residual and the added features given the
original features is sufficient, but not necessary, for equal MSE Bayes
risks. The exact criterion is equality of the conditional means.

\subsection{Working Distribution for Two-Stage Uncertainty}\label{appendix:theory:Variance Decomposition for Two-Stage Uncertainty Quantification}

Additive residual constructions also appear in multi-fidelity Gaussian process models such as ResGP \citep{xing2021residual}. FIRE's additive variance rule can be interpreted through a working
predictive distribution. Throughout this construction, the query $x$,
observed contexts, and inference configuration are fixed.
\begin{proposition}[Moments of the working distribution]
Let auxiliary draws have laws
\[
 U\sim q_\theta(\cdot\mid x,\mathcal D_{\mathrm{LF}}),\qquad
 V\sim q_\phi(\cdot\mid z_{\mathrm{aug}},\mathcal D_{\mathrm{aug}}).
\]
Under the independent auxiliary construction below, the working response
$\widetilde Y_{\mathrm{HF}}=U+V$ has moments
\[
 \mathbb E[\widetilde Y_{\mathrm{HF}}]
 =\mu_\theta(x)+\mu_\phi(z_{\mathrm{aug}}),\qquad
 \operatorname{Var}(\widetilde Y_{\mathrm{HF}})
 =\sigma_\theta^2(x)+\sigma_\phi^2(z_{\mathrm{aug}}).
\]
\end{proposition}
\begin{assumption}[Independent auxiliary construction]
The auxiliary variables $U$ and $V$ are drawn independently from the
specified working laws, and both have finite second moments.
\end{assumption}
\begin{proof}
Linearity gives the mean. The variance identity
$\operatorname{Var}(U+V)=\operatorname{Var}(U)+\operatorname{Var}(V)
+2\operatorname{Cov}(U,V)$ gives the variance, since independence implies
zero covariance. Zero covariance alone would suffice for this moment
identity, but not for the convolution in Appendix~C.10.
\end{proof}
FIRE's residual labels are $r_i=y_{\mathrm{HF},i}-\mu_\theta(x_i)$, and they not subtract a random draw of $U$. Thus the construction is an
approximation, not a derivation of the actual HF posterior or a proof
that actual LF and discrepancy errors are independent. In particular,
for common conditioning information $\mathcal I$ that determines
$\mu_\theta(x)$, the residual $r=Y_{\mathrm{HF}}-\mu_\theta(x)$ obeys
\[
 \operatorname{Var}(r\mid\mathcal I)
 =\operatorname{Var}(Y_{\mathrm{HF}}\mid\mathcal I).
\]
Adding the LF predictive variance is therefore an additional modeling
choice. The displayed working construction does not guarantee calibration.

\subsection{Information in Predictive Summaries}
We compare a predictive distribution with its mean, while distinguishing
that comparison from a baseline that also observes the input.
\begin{proposition}[Data processing for predictive summaries]
Let $F_{\mathrm{LF}}$ be a random predictive distribution with finite mean
$\mu_{\mathrm{LF}}$, a measurable function of $F_{\mathrm{LF}}$. Then
\[
 I(r;F_{\mathrm{LF}})\geq I(r;\mu_{\mathrm{LF}}).
\]
If these mutual informations are finite, their difference is
$I(r;F_{\mathrm{LF}}\mid\mu_{\mathrm{LF}})$, and equality holds exactly
when $r$ and $F_{\mathrm{LF}}$ are conditionally independent given
$\mu_{\mathrm{LF}}$.
\end{proposition}
Non-Gaussianity or heteroscedasticity alone does not imply strict
information gain. Moreover, FIRE's baseline also receives $X$. With LF
context and inference configuration fixed, $F_{\mathrm{LF}}$ is determined
by $X$, giving
\[
 I(r;F_{\mathrm{LF}}\mid X,\mu_{\mathrm{LF}})=0.
\]
Thus the comparison does not establish an information bottleneck for this
baseline. Even a positive information gain in a different setting need
not improve MSE Bayes risk unless it changes the conditional mean.
FIRE's finite quantile vector is a predictive summary, not a sufficient
statistic for an arbitrary distribution.

\subsection{Variance-Residual Coupling and Feature Representation}\label{Appendix:Theory:Variance-Residual Coupling and Mean-Only Misspecification}

Heteroscedastic GP models provide background for modeling input-dependent residual variability \citep{lazaro2010sparse}.

\begin{proposition}[Variance-residual coupling hypothesis]
For the learned residual $r=Y^{(T)}-\mu_\theta(X)$, we consider the
modeling hypothesis
\[
 \operatorname{Var}(r\mid X)=g(\sigma_\theta^2(X)),
\]
with LF context and inference configuration fixed and $g$ nonconstant
and nondecreasing. This is the hypothesis in Assumption 7.3, not a
consequence derived for every multi-fidelity problem.
\end{proposition}
\begin{corollary}[Conditioning at fixed LF context]
Let $Z_{\mathrm{mean}}=(X,\mu_\theta(X))$ and
$Z_{\mathrm{var}}=(Z_{\mathrm{mean}},\sigma_\theta^2(X))$.
With LF context and inference configuration fixed,
\[
 \mathcal L(r\mid Z_{\mathrm{mean}})
 =\mathcal L(r\mid Z_{\mathrm{var}})\quad\text{almost surely}.
\]
\end{corollary}
\begin{proof}
The variance summary is measurable with respect to $X$, so the two
feature sets generate the same sigma-algebra. Here $\mathcal L(r\mid Z)$
denotes a regular conditional law of $r$ given $Z$.
\end{proof}
The coupling hypothesis motivates an explicit feature for a constrained
learner. It does not imply misspecification of an unrestricted predictor
already using $X$, nor a necessary reduction in entropy or squared error.

\subsection{Quantile Representation and Finite Approximation}

Quantile functions provide a distributional representation \citep{gushchin2017integrated}. We give a direct proof and a finite-grid approximation bound below.

\begin{theorem}[Quantile representation]
Let $F$ be a CDF on $\mathbb R$ and define
$Q(u)=\inf\{y\in\mathbb R:F(y)\geq u\}$ for $u\in(0,1)$.
The full function $Q$ uniquely determines the law with CDF $F$.
For $u_k=k/(K+1)$ and
\[
 \widehat F_K(y)=\frac1K\sum_{k=1}^K\mathbf 1\{Q(u_k)\leq y\},
\]
we have
\[
 \sup_{y\in\mathbb R}|\widehat F_K(y)-F(y)|\leq\frac1{K+1}.
\]
\end{theorem}
\begin{proof}
The generalized-inverse identity is
$Q(u)\leq y\ \Longleftrightarrow\ u\leq F(y)$.
Hence, if $A\sim\operatorname{Unif}(0,1)$, then
$\mathbb P(Q(A)\leq y)=F(y)$, proving uniqueness.
For the finite-grid claim, put
$H_K(t)=K^{-1}\sum_{k=1}^K\mathbf 1\{k/(K+1)\leq t\}$.
The same identity gives $\widehat F_K(y)=H_K(F(y))$.
Counting grid points on each interval between successive $k/(K+1)$
gives $|H_K(t)-t|\leq1/(K+1)$ for all $t\in[0,1]$, establishing
the bound.
\end{proof}
FIRE uses nine predictive deciles as coarse summaries of the LF predictive
law. The result controls CDF approximation using exact quantiles of that
law, not its agreement with the true data-generating law. It does not
assert that nine quantiles identify an arbitrary distribution, determine
its skewness or kurtosis, or guarantee reduced HF prediction risk.

\subsection{Proof of Proposition 7.6}\label{appendix:theory:Quantile_Risk_Reduction} 

\begin{proof}
Set $m_0=\mathbb E[r\mid Z_{\mathrm{mv}}]$ and
$m_1=\mathbb E[r\mid Z_{\mathrm{quant}}]$. Since $r\in L^2$ and
$\sigma(Z_{\mathrm{mv}})\subseteq\sigma(Z_{\mathrm{quant}})$,
the projection identity in Appendix~C.2 gives
\[
 \mathcal R(Z_{\mathrm{mv}})-\mathcal R(Z_{\mathrm{quant}})
 =\mathbb E[(m_1-m_0)^2]\geq0.
\]
This is positive if and only if $\mathbb P(m_1\ne m_0)>0$.
For fixed LF context and inference configuration both feature sets
generate $\sigma(X)$, so $m_1=m_0$ almost surely.
\end{proof}
Non-Gaussianity is insufficient for strictness: if $r=E-1$, where
$E\sim\operatorname{Exp}(1)$ is independent of both feature sets, then
$r$ is non-Gaussian but $m_0=m_1=0$ and both Bayes risks are one.
Moreover, FIRE uses LF predictive quantiles, not residual quantiles;
a comparison between a residual median and mean cannot establish that
these LF features improve the residual conditional mean.

\subsection{Efficiency of Quantile Conditioning}

FIRE uses $\mathcal Q=\{0.1,0.2,\ldots,0.9\}$, the predictive quantile
set used in our implementation, as a compact distributional summary.
Composite quantile regression provides related motivation for using
multiple quantile levels, but it estimates regression coefficients through
a composite quantile-loss objective \citep{zou2008composite}. FIRE instead supplies LF predictive
quantiles as features to a frozen residual TFM. Consequently, efficiency
results for composite quantile regression do not provide an efficiency
bound for FIRE or establish the optimality of nine quantile features.

\subsection{Bayes-Risk Comparison of Feature Sets}

We compare unrestricted residual predictors and identify when their
risks agree. Let $r\in L^2$, and define
\[
 Z_{\mathrm{mean}}=(X,\mu_\theta(X)),\quad
 Z_{\mathrm{full}}=(Z_{\mathrm{mean}},\sigma_\theta^2(X),q_\theta(X)),
\]
where $q_\theta(X)$ denotes the quantile vector. Write
$m_0=\mathbb E[r\mid Z_{\mathrm{mean}}]$ and
$m_1=\mathbb E[r\mid Z_{\mathrm{full}}]$.
\begin{theorem}[Bayes-risk decomposition]
For $\mathcal R(m)=\mathbb E[(r-m)^2]$,
\[
 \mathcal R(m_0)-\mathcal R(m_1)
 =\mathbb E[(m_1-m_0)^2]
 =\mathbb E[\operatorname{Var}(m_1\mid Z_{\mathrm{mean}})].
\]
\end{theorem}
\begin{proof}
Apply Appendix~C.2. The last equality follows from
$\mathbb E[m_1\mid Z_{\mathrm{mean}}]=m_0$.
\end{proof}
The gap is positive exactly when the conditional means differ with
positive probability (heteroscedasticity and non-Gaussianity alone are
not sufficient).
\begin{corollary}[Fixed-context equality]
With LF context and inference configuration fixed, the Bayes risks using
$Z_{\mathrm{mean}}$ and $Z_{\mathrm{full}}$ are equal.
\end{corollary}
\begin{proof}
Both feature sets generate $\sigma(X)$.
\end{proof}
\begin{remark}[Finite-learner interpretation]
Equal oracle risks do not imply equal predictions from fitted learners.
Explicit summaries can change representation and finite-sample behavior.
Comparisons between FIRE and the evaluated baselines are empirical. This identity does not establish their ordering.
\end{remark}

\subsection{Working Convolution and Relation to Stacking}

Under the independent auxiliary construction in
Appendix~\ref{fireaudit:c3}, and when both component laws have densities,
$\widetilde Y_{\mathrm{HF}}=U+V$ has working density
\[
 p_{\mathrm{FIRE}}^{\mathrm{work}}(y)
 =\int_{\mathbb R}q_\theta(y-v\mid x,\mathcal D_{\mathrm{LF}})
             q_\phi(v\mid z_{\mathrm{aug}},\mathcal D_{\mathrm{aug}})\,dv.
\]
The query, contexts, and inference configuration are fixed here.
This construction has the moments returned by Algorithm 1; the algorithm
need not evaluate this integral. Zero covariance alone would not justify
the convolution law, and the construction is not asserted to equal the
true HF posterior.

The connection to predictive stacking is conceptual: FIRE uses LF
predictive summaries to condition a residual learner and combines
predictive moments. Algorithm 1 does not optimize the log score of this
convolved density. Optimality results for fitted linear-mixture weights
do not establish optimality of FIRE's working convolution.

\subsection{Empirical Validations of Distributional Conditioning on Classic Heteroscedastic Problems}\label{Appendix:Theory:empirical_proof}
To isolate the impact of FIRE's distributional conditioning, we benchmark against three classic heteroscedastic regression problems:
\begin{enumerate}
    \item \citet{goldberg1997regression}: Input data $x\in[0, 1]$; Base function $\mu(x)= 2\sin(2\pi x_i)$; Noise $\varepsilon(x)=\mathcal{N}(0,\sigma(x))$ (Gaussian) with $\sigma(x) = 0.5 + x$ that increases linearly.
    \item \citet{yuan2004doubly}: Input data $x\in[0, 1]$; Base function $\mu(x) = 2\left[\exp\left(-30(x - 0.25)^2\right) + \sin(\pi x^2)\right] - 2$; Noise $\varepsilon(x)=\mathcal{N}(0,\sigma(x))$ (Gaussian) with $\sigma(x) = e^{\sin(2\pi x_i)}$.
    \item \citet{williams1996using}: Input data $x\in[0, 1]$; Base function $\mu(x) = \sin(2.5x) \cdot \sin(1.5x)$; Noise $\varepsilon(x)=\mathcal{N}(0,\sigma(x))$ (Gaussian) with $\sigma(x) = 0.01 + 0.25(1 - \sin(2.5x))^2$.
\end{enumerate}
We adapted these single-fidelity benchmarks into a multi-fidelity (shown in Figure \ref{fig:theory_empirical}) setting by generating LF proxies with the base function, while HF observations retained the complex, input-dependent noise structures defined in the original literature~\citep{kersting2007most}. We perform a controlled ablation of the stacking strategy. We fix the model backbone, using either all GPs (FIRE-GP) or all TFMs (FIRE), and vary only the input to the residual learner. We compare standard mean-only stacking (without distribution info $\sigma^2, \{q^{(\tau)}\}$), where the residual model conditions only on the LF mean, against FIRE's distributional stacking, where the residual model conditions on the LF mean, variance, and quantiles.

Table \ref{tab:hetero_result} and \ref{tab:hetero_result_NLL} report the accuracy (NRMSE) and negative log-likelihood (NLL). Methods without distribution info $\sigma^2, \{q^{(\tau)}\}$ (mean-conditioned) capture the global trend and achieve comparable NRMSE. However, they perform poorly on NLL. For example, on the Williams function, the mean-only TFM yields an NLL of 2.34, while the distribution-conditioned TFM achieves 1.50. The mean-only approach remains overconfident in high-noise regions. Conditioning on variance and quantiles allows the residual model to detect heteroscedasticity and widen confidence intervals locally.

\begin{figure}[htbp!]
    \centering
    \includegraphics[width=0.8\linewidth]{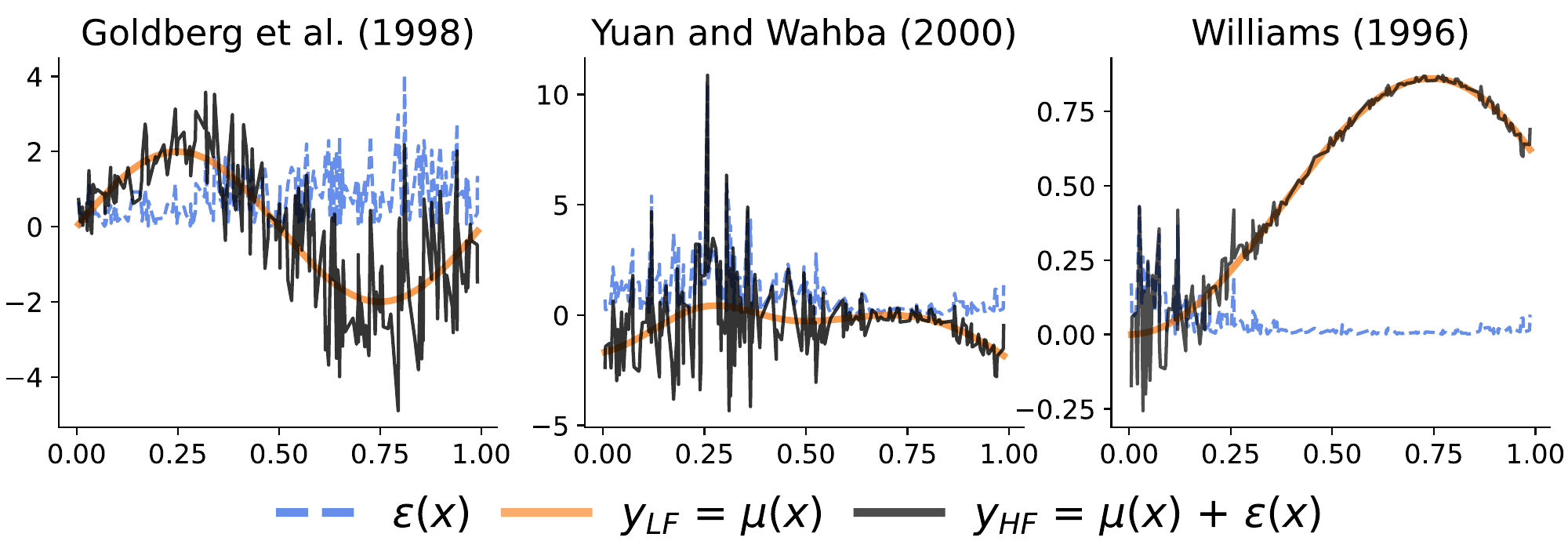}
    \caption{Visualizations of classic heteroscedastic regression benchmarks.}
    \label{fig:theory_empirical}
\end{figure}

\begin{table}[htbp!]
\centering
\caption{Predictive Accuracy (NRMSE) on Heteroscedastic Benchmarks (lower the better, \textbf{bold} is the best).}
\label{tab:hetero_result}
\begin{tabular}{lccc}
\toprule
  & \citet{goldberg1997regression} & \citet{yuan2004doubly} & \citet{williams1996using} \\ \midrule
FIRE-GP \textbf{without} distribution info $\sigma^2, \{q^{(\tau)}\}$ & 0.1933$\pm$0.0289 & 0.1844$\pm$0.0328 & 0.1110$\pm$0.0328  \\
FIRE-GP& 0.1867$\pm$0.0341 & 0.1834$\pm$0.0328 & 0.1082$\pm$0.0323 \\
FIRE \textbf{without} distribution info $\sigma^2, \{q^{(\tau)}\}$ & 0.1710$\pm$0.0317 & 0.1707$\pm$0.0117 &  0.1076$\pm$0.0320\\
FIRE & \textbf{0.1699$\pm$0.0322} & \textbf{0.1696$\pm$0.0105} & \textbf{0.1073$\pm$0.0321} \\
\bottomrule
\end{tabular}
\end{table}

\begin{table}[htbp!]
\centering
\caption{Uncertainty Calibration (NLL) on Heteroscedastic Benchmarks (lower the better, \textbf{bold} is the best).}
\label{tab:hetero_result_NLL}
\begin{tabular}{lccc}
\toprule
  & \citet{goldberg1997regression} & \citet{yuan2004doubly} & \citet{williams1996using} \\ \midrule
FIRE-GP \textbf{without} distribution info $\sigma^2, \{q^{(\tau)}\}$ & 5.4922$\pm$2.9140 & 9.2186$\pm$5.435 & 21.797$\pm$18.610  \\
FIRE-GP & 3.2145$\pm$0.7244 & 6.9235$\pm$3.456 & 17.366$\pm$12.476 \\
FIRE \textbf{without} distribution info $\sigma^2, \{q^{(\tau)}\}$ & 1.5759$\pm$0.1027 & 1.6433$\pm$0.1705 & 2.3486$\pm$3.3993 \\
FIRE & \textbf{1.5631$\pm$0.1031} & \textbf{1.6359$\pm$0.1884} & \textbf{1.5039$\pm$2.3797} \\
\bottomrule
\end{tabular}
\end{table}

\FloatBarrier
\newpage

\section{Benchmark Dataset Details}\label{Appendix:Section:Benchmark}

This section describes the benchmark datasets used in our experiments, organized by problem category. 
For each dataset, we detail the problem dimensionality, fidelity structure, and data generation procedure. 
A quantitative summary of input dimension, number of fidelities, and sample counts is provided in Table~\ref{Table:datasets}.

\subsection{Synthetic Datasets: }The implementations for the total synthetic problems are taken from
\begin{enumerate}
    \item \textbf{MF2} library \citep{vanRijn2020MF2} (11 problems): A Python library that contains a variety of synthetic problems for benchmarking multi-fidelity kriging and regression algorithms. We select the most commonly used benchmark problems for GP-based MFR: \texttt{bohachevsky, booth, borehole, branin, currin, forrester, hartmann6, himmelblau, park91a, park91b, six\_hump\_camelback}. (link: \url{https://mf2.readthedocs.io/en/latest/}, license: GPL-3.0 license, last accessed: January 22nd, 2026)
    \item \textbf{Emukit} library \citep{cutajar2019deep} (2 problems): A Python toolkit with code for multi-fidelity emulation, Bayesian optimization, and uncertainty-based decision making. We take the problem formulation of our \texttt{branin3f} and \texttt{hartmann3f} from the test functions module in this library. (link: \url{https://github.com/EmuKit/emukit/tree/main/emukit/test_functions}, license: Apache-2.0 license, last accessed: January 22nd, 2026)
    \item Equation \ref{eq:High_dim_functions} shows the two-fidelity high-dimensional problems we selected from the existing MFR literature for benchmarking deep-learning-model-based MFR algorithms \citep{Yi2024MFBNN,meng2020composite,shan2010Metamodeling}. In this research, we tested on 5 variants: $d\in\{10,20,30,40,50\}$.
    \begin{equation}\label{eq:High_dim_functions}
    \begin{aligned}
    f^{HF}(\mathbf{x}) &= \sum_{i=2}^{d} (2x_i^2 - x_{i-1})^2 + (x_1 - 1)^2 \\
        f^{LF}(\mathbf{x}) &= 0.8f^{HF}(\mathbf{x}) + \sum_{i=2}^{d} 0.4x_{i-1}x_i - 50
    \end{aligned}
    \end{equation}
\end{enumerate}

\subsection{Hyperparameter Optimization (HPO) Datasets: }

We select 6 HPO MF datasets (\texttt{adult, Fashion-MNIST, higgs, jasmine, vehicle, volkert}) from LCBench \citep{zimmer2021auto} (link: \url{https://github.com/automl/LCBench}, license: Apache-2.0 license, last accessed: January 22nd, 2026), which is commonly used by MF Bayesian optimization research \citep{rakotoarison2024ifbo,lee2025costsensitive}. The raw datasets have $d=7$ and contain 2000 learning curves with 51 epochs (fidelity level, $T=51$). The detail of each feature in $x$ is in Table \ref{tab:lcbench}.

\begin{table}[h]
\centering
\caption{The hyperparameters for LCBench datasets.}
\label{tab:lcbench}
\begin{tabular}{cccc}
\toprule
\textbf{Name} & \textbf{Type} & \textbf{Vaules} \\ \midrule
\texttt{batch\_size} & integer & $[2^4, 2^9]$ (log) \\
\texttt{learning\_rate} & continuous & $[10^{-4}, 10^{-1}]$ (log) \\
\texttt{max\_dropout} & continuous & $[0, 1]$ & \\
\texttt{max\_units} & integer & $[2^6, 2^{10}]$ (log) \\
\texttt{momentum} & continuous & $[0.1, 0.99]$ & \\
\texttt{max\_layers} & integer & $[1, 5]$ & \\
\texttt{weight\_decay} & continuous & $[10^{-5}, 10^{-1}]$ & \\ \bottomrule
\end{tabular}
\end{table}

We apply two processing steps to construct the multi-fidelity benchmark tasks:

\begin{enumerate}
    \item Fidelity selection: We define the fidelity levels using the top 5 epochs (46, 47, 48, 49, 50). Using the full 51 epochs is computationally infeasible for autoregressive GP baselines like NARGP, which train separate surrogates for each fidelity level. We reached the compute time limit when attempting to train on all 51 levels. We restrict the problem to 5 fidelities to ensure all algorithms converge within the benchmark time budget. %
    \item Data imbalance: We enforce a hierarchical sampling structure ($N_T \ll \dots \ll N_1$) to mimic extreme data scarcity. We fix the low-fidelity sample sizes at $N_1=1000$, $N_2=750$, $N_3=500$, and $N_4=250$. We vary the high-fidelity target set $N_5$ in $\{1000, 750, 500, 250\}$ to evaluate robustness. %
\end{enumerate}

\subsection{Physics-based Simulation Problems and Datasets: }

\subsubsection{Existing Analytical Parametric Datasets:}

We selected the two two-fidelity analytical problems (wing and beam) and the three-fidelity, mixed-integer material design problem (HOIP) from the code repository of GP+ \citep{Yousefpour2024GPPlus}, a Python library for generalized GP modeling (link: \url{https://github.com/Bostanabad-Research-Group/GP-Plus}, license: MIT license, last accessed: January 22nd, 2026).

\subsubsection{Existing Simulated Parametric Datasets: BlendedNet Multi-Fidelity Extension}

We use the BlendedNet Multi-Fidelity Extension dataset \citep{Sung2026BlendedNetMF}, which pairs the high-fidelity FUN3D SA-RANS lift/drag coefficients from BlendedNet \citep{Sung2025BlendedNet} with matched low-fidelity OpenVSP/VSPAERO vortex-lattice predictions evaluated on the same BWB geometries and flight conditions (link: \url{https://dataverse.harvard.edu/dataset.xhtml?persistentId=doi:10.7910/DVN/VJT9EP}, license: CC0 1.0, last accessed: January 22nd, 2026).

\subsubsection{Concrete Dataset}

The \textbf{Concrete} problem ($d{=}8$) is built from the experimental concrete compressive strength dataset of~\citet{Yeh1998Concrete} (link:\url{https://archive.ics.uci.edu/dataset/165/concrete+compressive+strength}, license: CC BY 4.0, last accessed: January 22nd, 2026), which reports 1030 mix designs with eight quantitative inputs (cement, slag, fly ash, water, superplasticizer, coarse aggregate, fine aggregate, and curing age) and measured compressive strength in MPa. We treat the experimental results as our high-fidelity data, and modeling an LF dataset with a physics-inspired multiplicative power-law model in the spirit of Abrams' law on log-transformed inputs (water–cement ratio, cement, slag, fly ash, superplasticizer, age) \citep{Abrams1919, Yeh2006}, similar to modern extensions of Abrams-type relationships for fly-ash concretes and reformed water–cement ratio laws \citep{Mondal2022}. 

\paragraph{High-Fidelity Data: }We use the Concrete Compressive Strength dataset of \citet{Yeh1998Concrete} from the UC Irvine Machine Learning Repository (link: \url{https://archive.ics.uci.edu/dataset/165/concrete+compressive+strength}, license: CC BY 4.0, last accessed: January 22nd, 2026). The dataset reports $N_{HF}{=}1030$ distinct concrete mix designs. Each row contains eight quantitative inputs: cement, blast furnace slag, fly ash, water, superplasticizer, coarse aggregate, fine aggregate, and curing age (in days), and a measured compressive strength target in MPa. This dataset is obtained experimentally and known to exhibit strong nonlinear interactions between mixture proportions and time.

\paragraph{Low-Fidelity Modeling: }We derive a physics-inspired empirical LF formulation by fitting a multiplicative power-law model in the spirit of Abrams' law. Classical Abrams formulations relate compressive strength $f_c$ to the water--cement ratio $(w/c)$ via a power law \citep{Abrams1919}. We generalize this idea and fit a log--log regression, consistent with modern extensions of Abrams-type relationships that incorporate age and supplementary cementitious materials \citep{Yeh2006, Mondal2022}, of the form
    \[
    \log f_c^{LF}(\mathbf{x})
    \;=\;
    \beta_0
    + \beta_1 \log(w/c)
    + \beta_2 \log(\text{cement})
    + \beta_3 \log(\text{slag})
    + \beta_4 \log(\text{fly ash})
    + \beta_5 \log(\text{superplasticizer})
    + \beta_6 \log(t),
\]
where $t$ is the curing age (days). This corresponds to the multiplicative
model
\[
    f_c^{LF}(\mathbf{x})
    = \exp(\beta_0)\,
      (w/c)^{\beta_1}
      (\text{cement})^{\beta_2}
      (\text{slag})^{\beta_3}
      (\text{fly ash})^{\beta_4}
      (\text{superplasticizer})^{\beta_5}
      t^{\beta_6},
\]
with coefficients $\{\beta_i\}$ estimated by ordinary least squares on the
log-strength. For the Concrete surrogate, the fitted values (rounded to three significant figures) are
\[
\begin{aligned}
\beta_0 &= 2.56,\\
\beta_1 &= -0.815 \quad (\log(w/c)),\\
\beta_2 &= -0.0380 \quad (\log \text{cement}),\\
\beta_3 &= 0.0161 \quad (\log \text{slag}),\\
\beta_4 &= 0.00231 \quad (\log \text{fly ash}),\\
\beta_5 &= 0.0148 \quad (\log \text{superplasticizer}),\\
\beta_6 &= 0.292 \quad (\log \text{age}).
\end{aligned}
\]
We then define $y^{LF}(\mathbf{x}) = f_c^{LF}(\mathbf{x})$ and
$y^{HF}(\mathbf{x})$ as the measured compressive strength.

\subsubsection{Car Dataset}

We present the \textbf{Car} dataset ($d{=}23$), a multi-fidelity automotive aerodynamics dataset that we construct from the DrivAerNet dataset~\citep{elrefaie2025drivaernet}. (link: \url{https://github.com/Mohamedelrefaie/DrivAerNet}, license: Attribution-NonCommercial 4.0 International, last accessed: January 22, 2026)

\paragraph{High-Fidelity Data: }The high-fidelity labels $y^{HF}$ are the DrivAerNet steady 3D RANS drag coefficients. Each input $x$ consists solely of parametric shape controls such as ramp and diffusor angles, trunklid angle, side-mirror rotation and placement, rear-window and windscreen inclination, overall car length and width, roof height, and green-house angle. The external flow condition is held fixed across all cases (constant freestream speed, Reynolds number, and boundary setup). 

\paragraph{Low-Fidelity Simulation: }Using the same design parameter as the high-fidelity data, our low-fidelity labels $y^{LF}$ are obtained from a quasi-2D OpenFOAM pipeline: we remove wheel geometry from the 3D mesh, form a 2D silhouette projection, generate a refined 2D mesh with Gmsh, extrude it into a thin 3D slab, and run incompressible RANS (\texttt{simpleFoam}) to estimate drag (Figure~\ref{fig:car2d-pipeline}). This new paired HF--LF CAR dataset is one of the main contributions of our benchmark suite and is designed to mimic a realistic but geometry-only car drag prediction problem.

We present the details of the quasi-2D LF CFD pipeline used to generate $C_D^{LF}$ for the CAR dataset described in the main text. For each parametric geometry, we start from the DrivAerNet STL, construct a 2D silhouette, generate a quasi-2D mesh, and run a steady incompressible RANS simulation in OpenFOAM. The overall workflow is illustrated in Fig.~\ref{fig:car2d-pipeline}. Note that the mesh panel shows a zoomed-in region near the car, while the actual far-field domain extends several vehicle lengths upstream and downstream.

\begin{figure}[ht]
  \centering
  \includegraphics[width=0.8\linewidth]{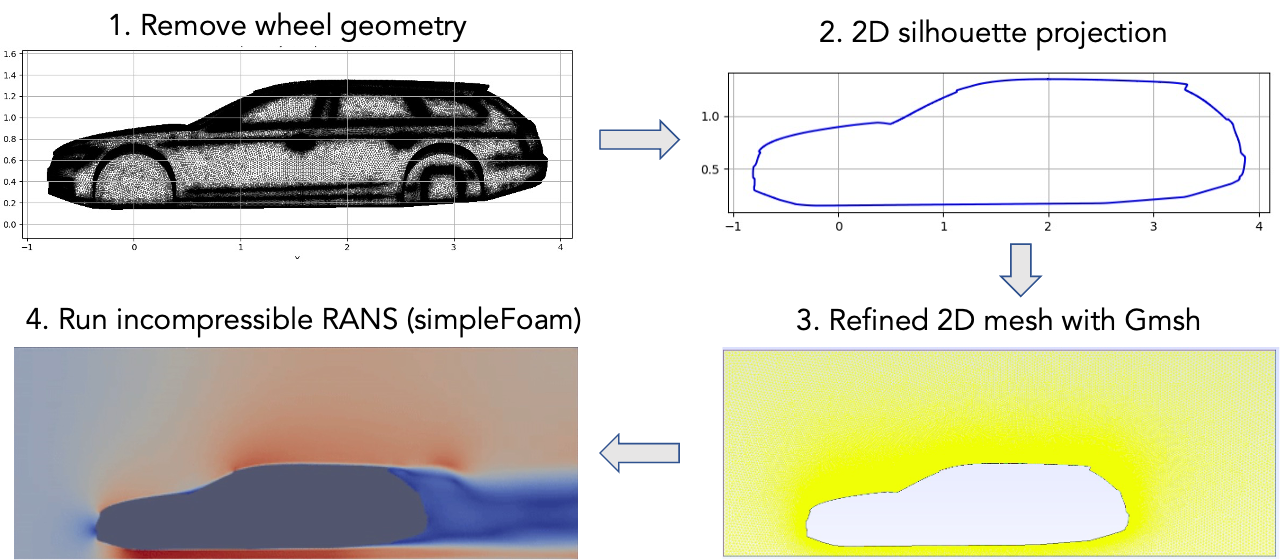}
  \caption{Quasi-2D low-fidelity pipeline for the CAR dataset.
  Top-left: projected mid-plane view of the 3D DrivAer-style geometry after removing wheels.
  Top-right: cleaned 2D silhouette obtained from the projected points.
  Bottom-right: Gmsh-generated quasi-2D mesh (shown as a zoom-in around the car; the full domain extends several car lengths upstream/downstream and vertically) with refinement concentrated near the body.
  Bottom-left: sample incompressible RANS solution from the \texttt{simpleFoam} run used to compute the LF drag coefficient.}
  \label{fig:car2d-pipeline}
\end{figure}

\begin{enumerate}
    \item \textbf{Flow conditions and fluid properties.} All LF simulations are run at the same freestream speed and fluid properties as the HF DrivAerNet cases. We prescribe a uniform inlet velocity $U_\infty = 30~\mathrm{m/s}$ for air with constant density and kinematic viscosity $\rho_\infty = 1.184~\mathrm{kg/m^3}, 
\;
\nu = 1.56\times 10^{-5}~\mathrm{m^2/s}.$

This yields Reynolds numbers comparable to the HF setup for a typical passenger-car length scale. The reference area used in the LF drag normalization is defined from the projected frontal height of the car, $h_{\mathrm{front}}$, multiplied by the extrusion thickness, $A_{\mathrm{ref}}^{LF}
=
h_{\mathrm{front}}\,H_{\mathrm{extrude}},
\;
H_{\mathrm{extrude}} = 10^{-3}~\mathrm{m}.$
We keep HF and LF drag coefficients in their respective native normalizations and let the surrogate models learn any scale differences between fidelities.
    \item \textbf{Geometry processing and 2D silhouette.}
Starting from the 3D STL surface of the car body, we first remove all wheel geometry before further processing. This avoids artificial lobes and unphysical wheel wakes in the 2D approximation, and focuses the LF model on the main body shape. The remaining surface is sampled densely: we collect all mesh vertices and an additional $3\times 10^5$ random surface points, projected onto the $x$--$z$ plane. From the resulting point cloud we construct a watertight 2D silhouette by computing an $\alpha$-shape with $\alpha = 0.02 \times \bigl(\text{diagonal of the 3D bounding box}\bigr),$ which scales automatically with the geometry size. The polygon boundary is then cleaned by removing near-duplicate points and very short edges, yielding a smooth, single closed curve that captures the outer hull of the vehicle. This step corresponds to the top-right panel of Fig.~\ref{fig:car2d-pipeline}.
    \item \textbf{Computational domain in 2D.}
Let $L$ denote the car length and $x_{\mathrm{car,min}}$, $x_{\mathrm{car,max}}$ the minimum and maximum $x$-coordinates of the silhouette. We embed the car in a rectangular 2D far-field domain defined by
\[
\begin{aligned}
x_\mathrm{min} &= x_{\mathrm{car,min}} - 3L,\\
x_\mathrm{max} &= x_{\mathrm{car,max}} + 10L,\\
z_\mathrm{min} &= 0, \qquad
z_\mathrm{max} = z_{\mathrm{car,max}} + 5~\mathrm{m}.
\end{aligned}
\]
For a nominal passenger car with $L \approx 4.5$--$4.6~\mathrm{m}$ this corresponds to roughly $13$--$14~\mathrm{m}$ of upstream length, $45~\mathrm{m}$ downstream, and a top boundary several meters above the roof. The mesh close-up in Fig.~\ref{fig:car2d-pipeline} shows only a zoomed-in subregion near the body; the full computational box is significantly larger in all directions.

    \item \textbf{Gmsh meshing and 2.5D extrusion.}
The 2D silhouette and far-field box are meshed using a fully automated Gmsh pipeline. We employ a distance-based size field with two nominal element sizes: $h_{\text{fine}} = 0.01~\mathrm{m},\; h_{\text{coarse}} = 0.1~\mathrm{m}.$
Cells near the car boundary are refined towards $h_{\text{fine}}$, while elements in the far field gradually transition towards $h_{\text{coarse}}$ with additional upstream/downstream growth limits. The resulting 2D meshes contain on the order of $\mathcal{O}(10^5)$ triangular/quadrilateral cells per case, providing a reasonably resolved near-body region for trend-accurate drag without explicit $y^+$ targeting.

To run OpenFOAM, the 2D region is extruded into 3D by a small thickness, $H_{\mathrm{extrude}} = 10^{-3}~\mathrm{m}$, with a single prism layer, yielding a quasi-2D (``2.5D'') volume mesh. The extrusion exists solely to satisfy OpenFOAM's 3D requirements; we explicitly enforce flow invariance across the thin spanwise direction by using empty boundary conditions on the extruded faces. During meshing, we automatically label the boundary patches as
\texttt{carWall}, \texttt{inlet}, \texttt{outlet}, \texttt{ground}, \texttt{freestream} (top), and \texttt{front}/\texttt{back} (extruded faces).

\item \textbf{Mesh conversion and boundary conditions.}
The Gmsh mesh is converted to OpenFOAM format using \texttt{gmshToFoam}, and quality is checked with \texttt{checkMesh} (non-orthogonality, skewness, and boundary consistency). Custom scripts adjust the patch names to match the simulation setup. Boundary conditions are as follows:
\begin{itemize}
    \item \textbf{Inlet (\texttt{inlet}):} fixed velocity $\boldsymbol{U} = (30, 0, 0)~\mathrm{m/s}$, prescribed turbulence fields $k$ and $\omega$ consistent with the inlet turbulence intensity.
    \item \textbf{Outlet (\texttt{outlet}):} zeroGradient for $\boldsymbol{U}$ and turbulence fields, fixedValue $p=0$ for pressure.
    \item \textbf{Ground (\texttt{ground}):} no-slip condition for $\boldsymbol{U}$, standard wall treatment for $k$ and $\omega$.
    \item \textbf{Car body (\texttt{carWall}):} no-slip condition on the silhouette, with wall functions for turbulence as in a standard $k$--$\omega$ SST setup.
    \item \textbf{Top boundary (\texttt{freestream}):} slip condition for velocity (zero normal velocity, zero shear), zeroGradient for pressure.
    \item \textbf{Front/back (\texttt{front}, \texttt{back}):} \texttt{empty} boundaries, enforcing 2D behavior across the thin extrusion direction.
\end{itemize}
This combination encodes the modeling assumption that the flow is invariant in the spanwise direction and can be approximated as 2D, while still using a 3D finite-volume solver.

\item 
\textbf{Numerical schemes and linear solvers.}
All simulations use the steady-state \texttt{simpleFoam} solver with an incompressible RANS formulation and a $k$--$\omega$ SST turbulence model. The discretization schemes in \texttt{fvSchemes} are:
\begin{itemize}
    \item \textbf{Time derivative:} \texttt{ddtSchemes\{\ default steadyState;\}}.
    \item \textbf{Gradients:} Gauss linear with cell-limited gradients for $\nabla \boldsymbol{U}$, $\nabla k$, and $\nabla \omega$.
    \item \textbf{Convective terms:} first-order Gauss upwind for $\texttt{div(phi,U)}$, $\texttt{div(phi,k)}$, and $\texttt{div(phi,omega)}$ for robustness across many geometries.
    \item \textbf{Diffusive/viscous terms:} Gauss linear with corrected surface-normal gradients (e.g., for $\nabla\cdot(\nu_{\mathrm{eff}}\nabla \boldsymbol{U})$ and Laplacians).
\end{itemize}
The linear solvers and SIMPLE controls in \texttt{fvSolution} are:
\begin{itemize}
    \item Pressure $p$: \texttt{GAMG} with Gauss--Seidel smoother, absolute tolerance $10^{-6}$ and relative tolerance $0.1$ for intermediate iterations; a final solve with \texttt{relTol = 0} ensures full convergence.
    \item Velocity and turbulence fields $(\boldsymbol{U},k,\omega)$: \texttt{smoothSolver} with symmetric Gauss--Seidel smoothing, absolute tolerance $10^{-6}$ and relative tolerance $0.1$ for intermediate steps, followed by final solves with zero relative tolerance.
    \item SIMPLE algorithm: \texttt{nNonOrthogonalCorrectors = 2} non-orthogonal corrections per iteration, with under-relaxation factors of $0.3$ for $p$ and $0.5$ for $\boldsymbol{U}$, $k$, and $\omega$.
\end{itemize}
In practice, we iterate until residuals have dropped by several orders of magnitude and the forceCoeffs-reported drag has reached a stable plateau.

\item \textbf{Drag extraction and computational cost.}
Drag and lift coefficients are obtained using OpenFOAM's \texttt{forceCoeffs} function object applied to the \texttt{carWall} patch:
\[
C_D^{LF}
=
\frac{1}{\tfrac{1}{2}\rho_\infty U_\infty^2 A_{\mathrm{ref}}^{LF}}
\bigl(\text{streamwise force on \texttt{carWall}}\bigr),
\]
The function object is evaluated every time step, and the final drag coefficient is taken from the plateau region of the force history.

Each LF simulation typically requires on the order of $10$ minutes of wall-clock time on a node with $\sim 16$ CPU cores, reflecting the modest mesh size and relatively simple numerics. In contrast, the HF DrivAerNet simulations rely on significantly larger 3D meshes and more expensive numerics, requiring on the order of hundreds of CPU-hours per case. The resulting CAR dataset therefore provides a realistic multi-fidelity pairing: a costly, high-accuracy 3D RANS drag coefficient $C_D^{HF}$ and a much cheaper, trend-accurate quasi-2D estimate $C_D^{LF}$ for the same parametric car geometries.

\end{enumerate}

\subsection{Data Sampling Imbalance Protocol.} As mentioned in \ref{Section:5.2:Benchmark Datasets Design}, we systematically evaluate performance under data imbalance regimes by fixing the LF sample size and  $N_{\text{LF}}$ and varying the HF budget $N_{\text{HF}}$ across six imbalance ratios: $\{2\%, 4\%, 5\%, 10\%, 20\%, 25\%\}$. For synthetic functions, we sampled the data using Latin hypercube sampling and set $N_{\text{LF}}$ = 200 for problems with dimension $d<10$, and $N_{\text{LF}}$ = 2000 for $d\geq10$ based on the curse of dimensionality exponential scaling laws. For LCBench HPO problems, we fixed the $\{N_1, N_2, N_3, N_4\} = \{1000,750,500,250\}$ following how the exising literature partitioned their data imbalance for problems with more than 2 fidelity \citep{Yousefpour2024GPPlus}. For engineering problems, we set the $N_{\text{LF}}$ based on the existing dataset size. The $N_{\text{test}}$ for all problems is set to be $\frac{N_{\text{LF}}}{2}$.

\begin{table}[ht]
  \caption{List of datasets}
  \label{Table:datasets}\footnotesize
  \centering
  \begin{tabular}{cccccc}
    \toprule
    Dataset & Data Source & \makecell{Dimension\\(Features, $d$)} & \makecell{Fidelity\\Level ($T$)}& \makecell{$\{N_{t=1},\, N_{t=2},\,$\\$\cdots,\, N_{t=T-1}\}$} & $N_{T}$ \\ %
    \midrule
    Bohachevsky & Synthetic (MF2) & 2 & 2 & \{200\} & \{4, 8, 10, 20, 40, 50\} \\
    Booth & Synthetic (MF2) & 2 & 2  & \{200\} & \{4, 8, 10, 20, 40, 50\} \\
    Borehole & Synthetic (MF2) & 8 & 2  & \{200\} & \{4, 8, 10, 20, 40, 50\} \\
    Branin & Synthetic (MF2) & 2 & 2  & \{200\} & \{4, 8, 10, 20, 40, 50\} \\
    Currin & Synthetic (MF2) & 2 & 2  & \{200\} & \{4, 8, 10, 20, 40, 50\} \\
    Forrester & Synthetic (MF2) & 1 & 2  & \{200\} & \{4, 8, 10, 20, 40, 50\} \\
    Hartmann & Synthetic (MF2) & 6 & 2  & \{200\} & \{4, 8, 10, 20, 40, 50\} \\
    Himmelblau & Synthetic (MF2) & 2 & 2  & \{200\} & \{4, 8, 10, 20, 40, 50\} \\
    Park91a & Synthetic (MF2) & 4 & 2  & \{200\} & \{4, 8, 10, 20, 40, 50\} \\
    Park91b & Synthetic (MF2) & 4 & 2  & \{200\} & \{4, 8, 10, 20, 40, 50\} \\
    Six-Hump Camelback & Synthetic (MF2) & 2 & 2  & \{200\} & \{4, 8, 10, 20, 40, 50\} \\
    HD 10 & Synthetic (BNN) & 10 & 2  & \{2000\} & \{40, 80, 100, 200, 400, 500\} \\
    HD 20 & Synthetic (BNN) & 20 & 2  & \{2000\} & \{40, 80, 100, 200, 400, 500\} \\
    HD 30 & Synthetic (BNN) & 30 & 2  & \{2000\} & \{40, 80, 100, 200, 400, 500\} \\
    HD 40 & Synthetic (BNN) & 40 & 2  & \{2000\} & \{40, 80, 100, 200, 400, 500\} \\
    HD 50 & Synthetic (BNN) & 50 & 2  & \{2000\} & \{40, 80, 100, 200, 400, 500\} \\
    branin3f & Synthetic (GP+) & 2 & 3  & \{200, 50\} & \{40, 80, 100, 200, 400, 500\} \\
    hartmann3f & Synthetic (GP+) & 3  & 3 & \{200, 50\} & \{40, 80, 100, 200, 400, 500\} \\
    adult & LCBench & 7  & 5 & \{1000, 750, 500, 250\} & \{20, 40, 50, 100, 200, 250\} \\
    Fashion-MNIST & LCBench & 7  & 5  & \{1000, 750, 500, 250\} & \{20, 40, 50, 100, 200, 250\} \\
    Higgs & LCBench & 7  & 5 & \{1000, 750, 500, 250\} & \{20, 40, 50, 100, 200, 250\} \\
    Jasmine & LCBench & 7  & 5  & \{1000, 750, 500, 250\} & \{20, 40, 50, 100, 200, 250\} \\
    Vehicle & LCBench & 7  & 5 & \{1000, 750, 500, 250\} & \{20, 40, 50, 100, 200, 250\} \\
    Volkert & LCBench & 7  & 5 & \{1000, 750, 500, 250\} & \{20, 40, 50, 100, 200, 250\} \\
    Beam & Engineering & 5 & 2  & \{200\} & \{4, 8, 10, 20, 40, 50\} \\
    Wing & Engineering & 10 & 2  & \{2000\} & \{40, 80, 100, 200, 400, 500\} \\
    HOIP & Engineering & 3 & 3  & \{100, 25\} & \{2, 4, 5, 10, 20, 25\} \\
    Concrete & Engineering (Our) & 8 & 2  & \{200\} & \{4, 8, 10, 20, 40, 50\} \\
    BWB-CD & Engineering (Our) & 14 & 2  & \{1000\} & \{20, 40, 50, 100, 200, 250\} \\
    BWB-CL & Engineering (Our) & 14 & 2  & \{1000\} & \{20, 40, 50, 100, 200, 250\} \\
    Car & Engineering (Our) & 23 & 2  & \{500\} & \{10, 20, 25, 50, 100, 125\} \\

    \bottomrule
  \end{tabular}
\end{table}

\section{Algorithms Selection and Implementation Details}\label{Appendix:Section:Algorithm_Implementation}

\subsection{Baseline Algorithms}\label{subection:GP}

We benchmark FIRE against eight state-of-the-art multi-fidelity regression methods, using GPU-accelerated PyTorch implementations whenever available to ensure a fair and consistent comparison with TFMs run on GPU.

\begin{enumerate}
    \item AR(1)~\citep{Kennedy2000}: Since the original paper does not come with a public code repository, we follow the latest PyTorch implementation from the ContinuAR paper~\citep{xing2023continuar}'s supplementary material for our AR(1). %
    \item NARGP~\citep{Perdikaris2017}: We implemented the NARGP since the original code is not in PyTorch (link to original code: \url{https://github.com/Xiao-dong-Wang/Multifidelity-GP}, license: MIT license, last accessed: January 22nd, 2026). We used the latest implementation of NARGP, following both the original paper \citep{Perdikaris2017} and \citet{cutajar2019deep}, which expand NARGP to more than two fidelities by sequentially modeling each fidelity. 
    \item ResGP~\citep{xing2021residual}: Since the original paper does not come with a public code repository, we follow the latest PyTorch implementation from the ContinuAR paper~\citep{xing2023continuar}'s supplementary material for our ResGP.
    \item MisoKG~\citep{poloczek2017multi}: We take the open-source MisoKG code of its surrogate from BoTorch’s GitHub repository (link: \url{https://botorch.org/docs/tutorials/discrete_multi_fidelity_bo/}, license: MIT license, last accessed: January 22nd, 2026).
    \item ContinuAR~\citep{xing2023continuar}:  We follow the implementation from the original paper's supplementary material as our ContinuAR code.
    \item MFBNN~\citep{Yi2024MFBNN}: We take the code from the paper's open-source GitHub repository. (link: \url{https://github.com/bessagroup/mfbml}, license: BSD 3-Clause License, last accessed: January 22nd, 2026).
    \item MFRNP~\citep{niu24MFRNP}: We take the code from the paper's open-source GitHub repository. (link: \url{https://github.com/Rose-STL-Lab/MFRNP}, license: MIT license, last accessed: January 22nd, 2026).
    \item FIRE-GP: We implemented FIRE-GP using BoTorch's \texttt{SingleTaskGP} since the model support the API for accessing mean, variance, and quantiles through posterior estimation (link: \url{https://botorch.org/docs/getting_started}, license: MIT license, last accessed: January 22nd, 2026). This API analytically evaluates the inverse cumulative distribution function (ICDF) of the marginal Gaussian posterior at each candidate point for computing quantiles.
\end{enumerate}

\subsection{Tabular Foundation Models}\label{subection:TFM}

We use all tabular foundation models in a zero-shot setting, without any task-specific fine-tuning or hyperparameter tuning, to isolate their out-of-the-box generalization performance in multi-fidelity regression.

\begin{enumerate}
    \item TabPFNv2.5~\citep{grinsztajn2025tabpfn25advancingstateart}: We use the default pretrained checkpoint of TabPFNv2.5, \texttt{tabpfn-v2.5-regressor-v2.5\_default.ckpt}, released by the authors (Prior Labs) (link: \url{https://huggingface.co/Prior-Labs/tabpfn_2_5}, license: Prior Labs License, last accessed: January 22nd, 2026). 

    \item TabPFNv2~\citep{hollmann2025accurate}: We use the default pretrained checkpoint of TabPFNv2, \texttt{tabpfn-v2-regressor.ckpt}, released by the authors (Prior Labs) (link: \url{https://huggingface.co/Prior-Labs/TabPFN-v2-reg}, license: Prior Labs License, last accessed: January 22nd, 2026).

    \item RealTabPFN~\citep{garg2025real}: We use the default pretrained checkpoint of RealTabPFN, \texttt{tabpfn-v2.5-regressor-v2.5\_real.ckpt}, which is trained on a mixture of synthetic and real-world tabular datasets (link: \url{https://huggingface.co/Prior-Labs/tabpfn_2_5}, license: Prior Labs License, last accessed: January 22nd, 2026).
    \item Mitra~\citep{zhang2025mitra}: We use the pretrained \texttt{mitra-regressor} model released by AutoGluon without fine-tuning (link: \url{https://huggingface.co/autogluon/mitra-regressor}, license: Prior Labs License, last accessed: January 22nd, 2026). 
\end{enumerate}

\paragraph{Accessing statistical information from TFMs: } All TabPFN models suport the API for accesing mean, variance, and quantile through their \texttt{regressor.py}. TabPFN approximates the posterior predictive distribution $p(y|x, \mathcal{D})$ using a piecewise-constant Riemann distribution defined over fixed intervals (buckets), where the model outputs logits representing the probability mass of each bucket \citep{muller2021transformers,hollmann2023tabpfn,hollmann2025accurate}. Quantiles are analytically derived by computing the discrete cumulative distribution function (CDF) from these logits and performing linear interpolation between bucket borders to solve for the inverse CDF.

Mitra does not support the API for accessing quantile. To extract uncertainty estimates from the Mitra model, we implement a split conformal prediction wrapper that calibrates the model's outputs using a held-out portion of the training data. Specifically, we calculate the empirical percentiles of the prediction residuals ($y_{true} - \hat{y}$) on this calibration set to determine constant additive offsets for each desired quantile. At inference time, these offsets are added to the model's mean prediction to generate the final quantile estimates.

\section{Performance Results Per Datasets}\label{Section:FullResults}

This section provides the complete breakdown of predictive performance for all 31 datasets included in our benchmark suite. We visualize the impact of data scarcity by plotting Normalized Root Mean Squared Error (NRMSE), Negative Log Likelihood (NLL), and Coefficient of Determination ($R^2$) against varying high-fidelity to low-fidelity sample ratios. The figures below are organized by domain to facilitate detailed comparison: Figures \ref{fig:Engineering_single_simulated}-\ref{fig:Engineering_single_parametric} cover Physics-based simulations, Figure \ref{fig:HPO_single_problems} details Hyperparameter Optimization (HPO) tasks, and Figures \ref{fig:Synthetic_single_problem_v1}-\ref{fig:Synthetic_single_problem_v3} display results for Synthetic functions. Across these domains, FIRE demonstrates robust generalization, maintaining stable performance rankings even as the high-fidelity budget is severely constrained.

\begin{figure}[ht]
  \centering
  \includegraphics[width=.72\linewidth]{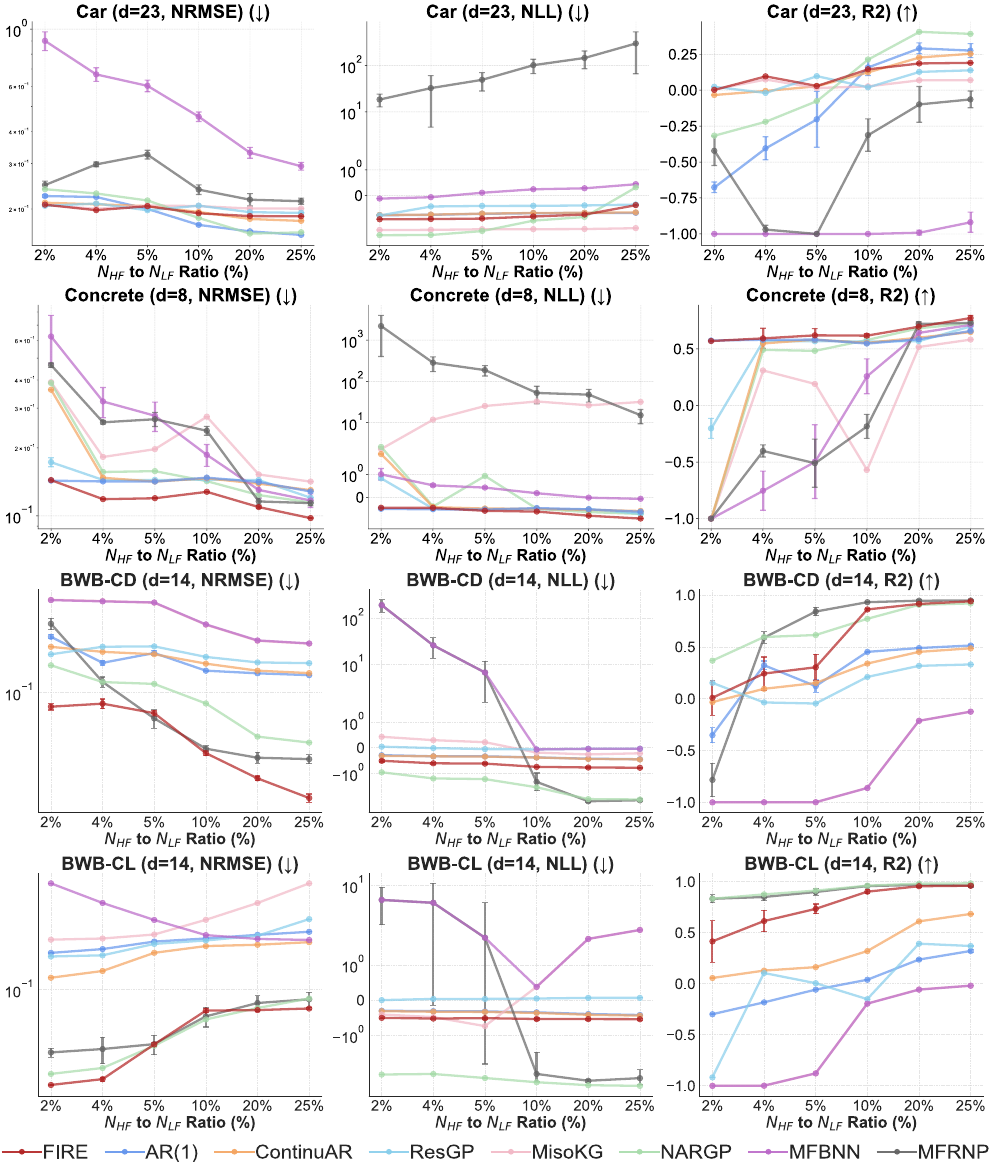}
  \caption{Result of Physics-based simulation problems (Car, Concrete, BWB).}
  \label{fig:Engineering_single_simulated}
\end{figure}

\begin{figure}[ht]
  \centering
  \includegraphics[width=.72\linewidth]{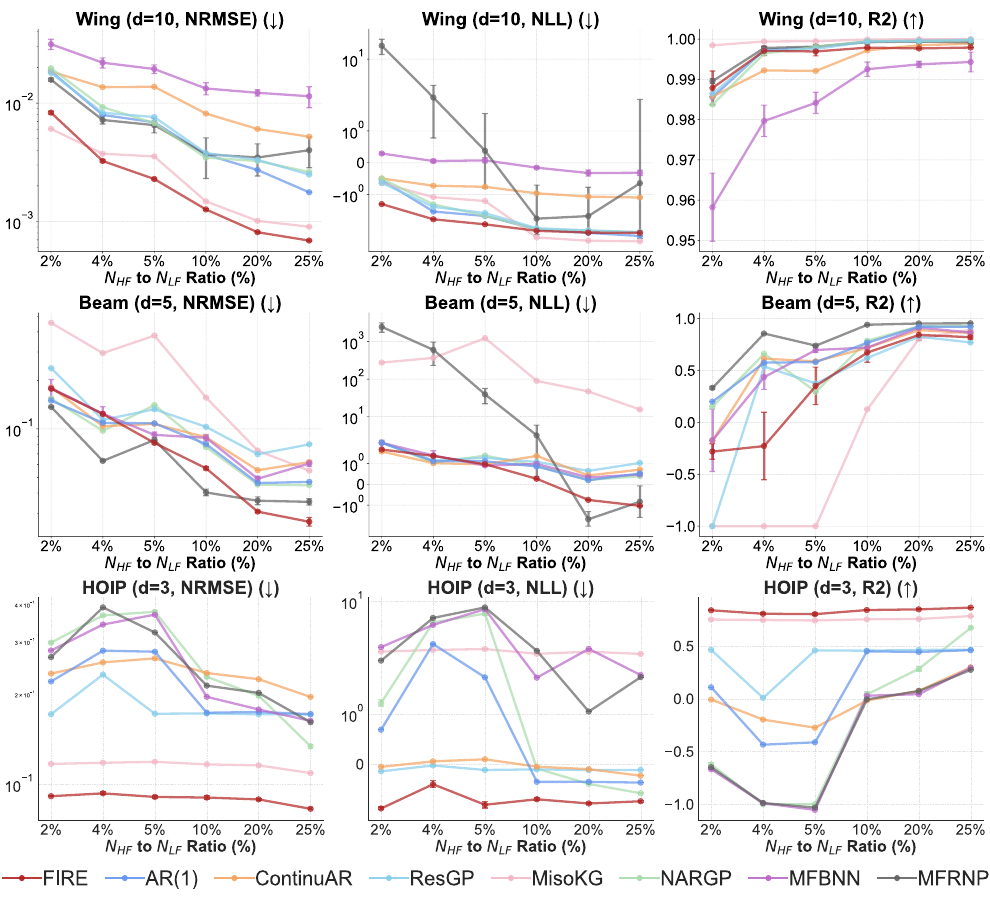}
  \caption{Result of Physics-based parametric problems (Wing, Beam, HOIP).}
  \label{fig:Engineering_single_parametric}
\end{figure}

\begin{figure}[ht]
  \centering
  \includegraphics[width=.72\linewidth]{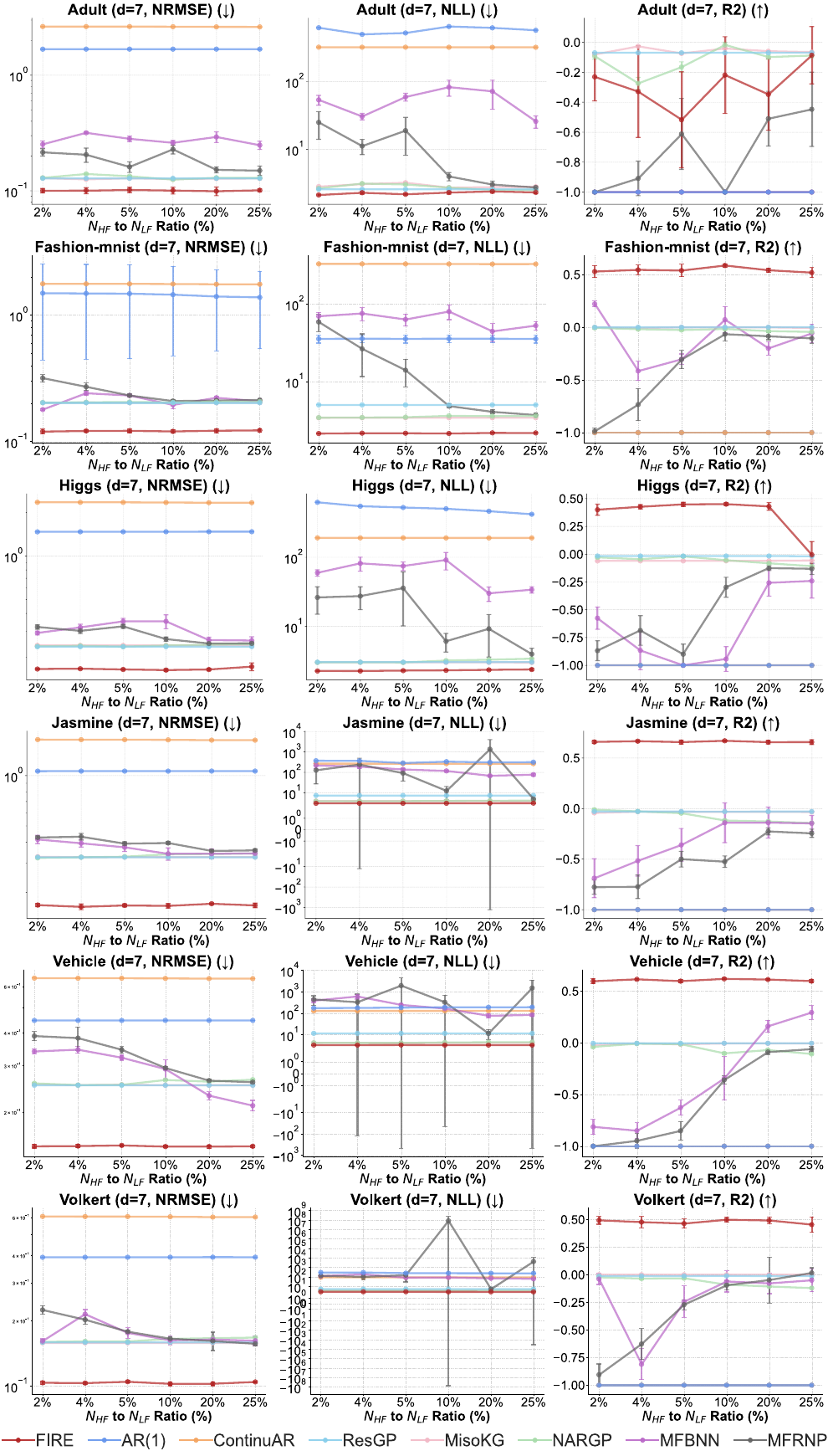}
  \caption{Result of Hyperparameter Optimization (HPO) problems from LCBench.}
  \label{fig:HPO_single_problems}
\end{figure}

\begin{figure}[ht]
  \centering
  \includegraphics[width=.72\linewidth]{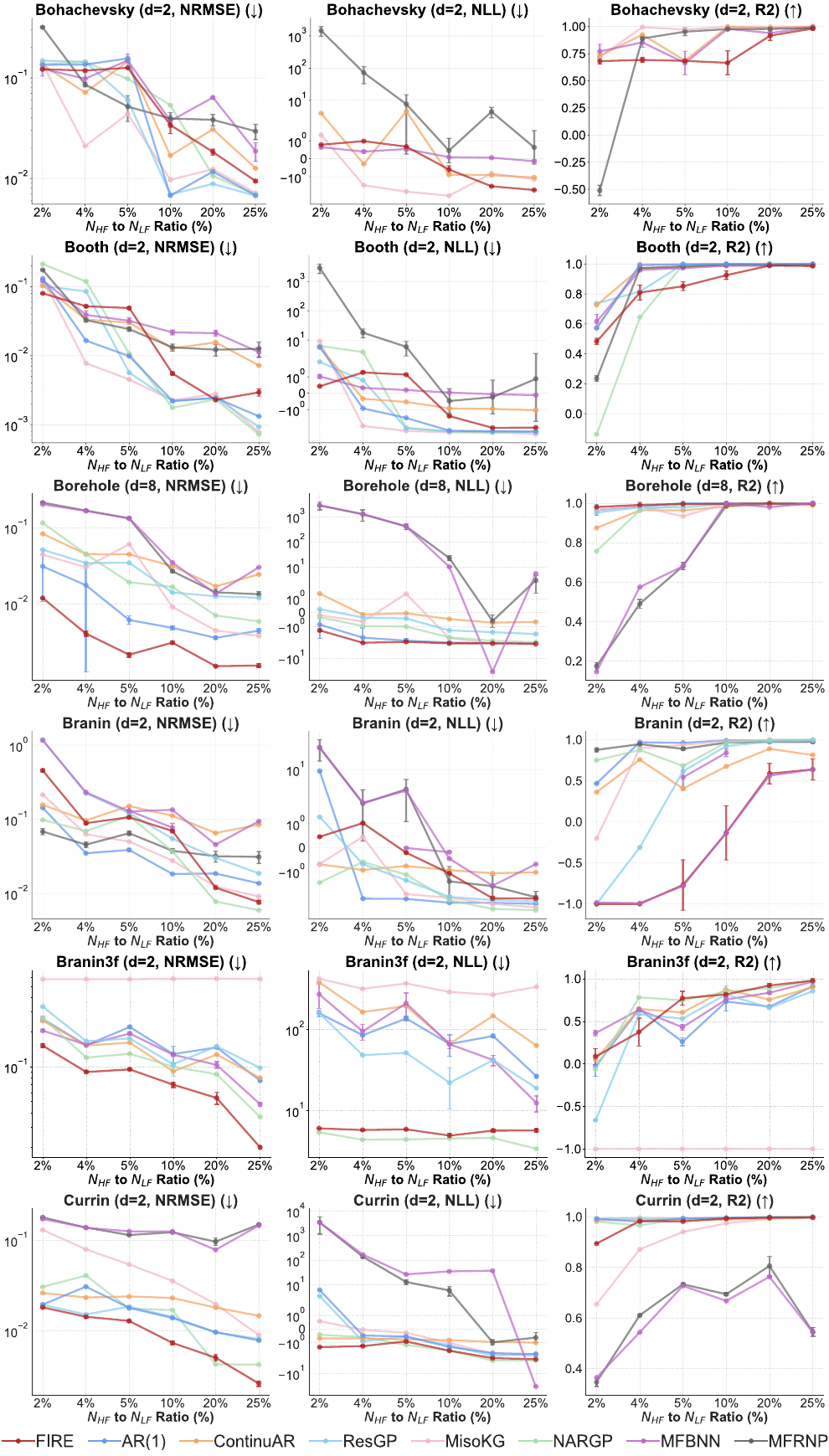}
  \caption{Result of synthetic problems (Bohachevsky, Booth, Borehole, Branin, Branin3f, Currrin).}
  \label{fig:Synthetic_single_problem_v1}
\end{figure}

\begin{figure}[ht]
  \centering
  \includegraphics[width=.72\linewidth]{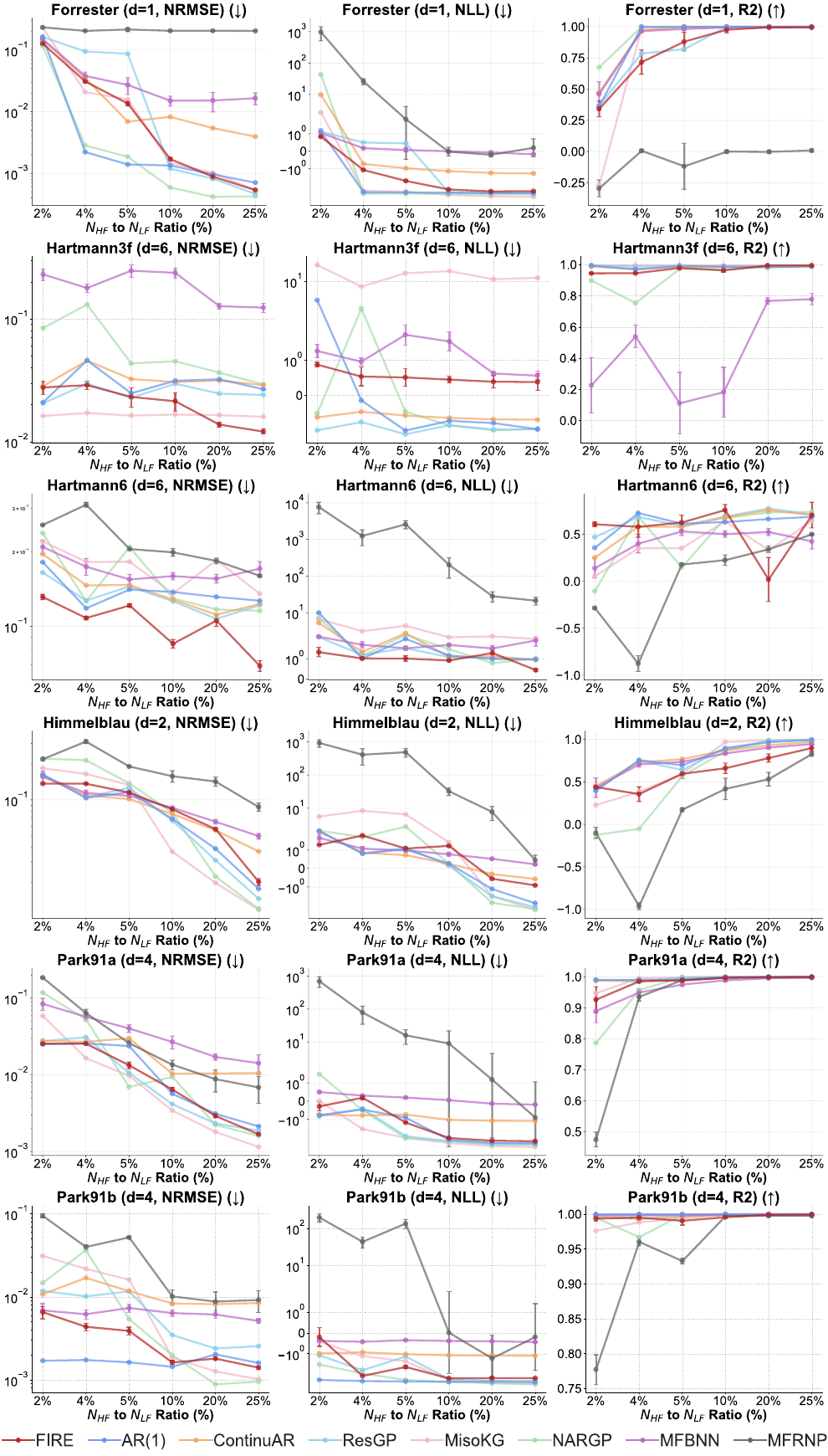}
  \caption{Result of synthetic problems (Forrester, Hartmann3f, Hartmann6, Himmelblau, Park91a, Park91b).}
  \label{fig:Synthetic_single_problem_v2}
\end{figure}

\begin{figure}[ht]
  \centering
  \includegraphics[width=.72\linewidth]{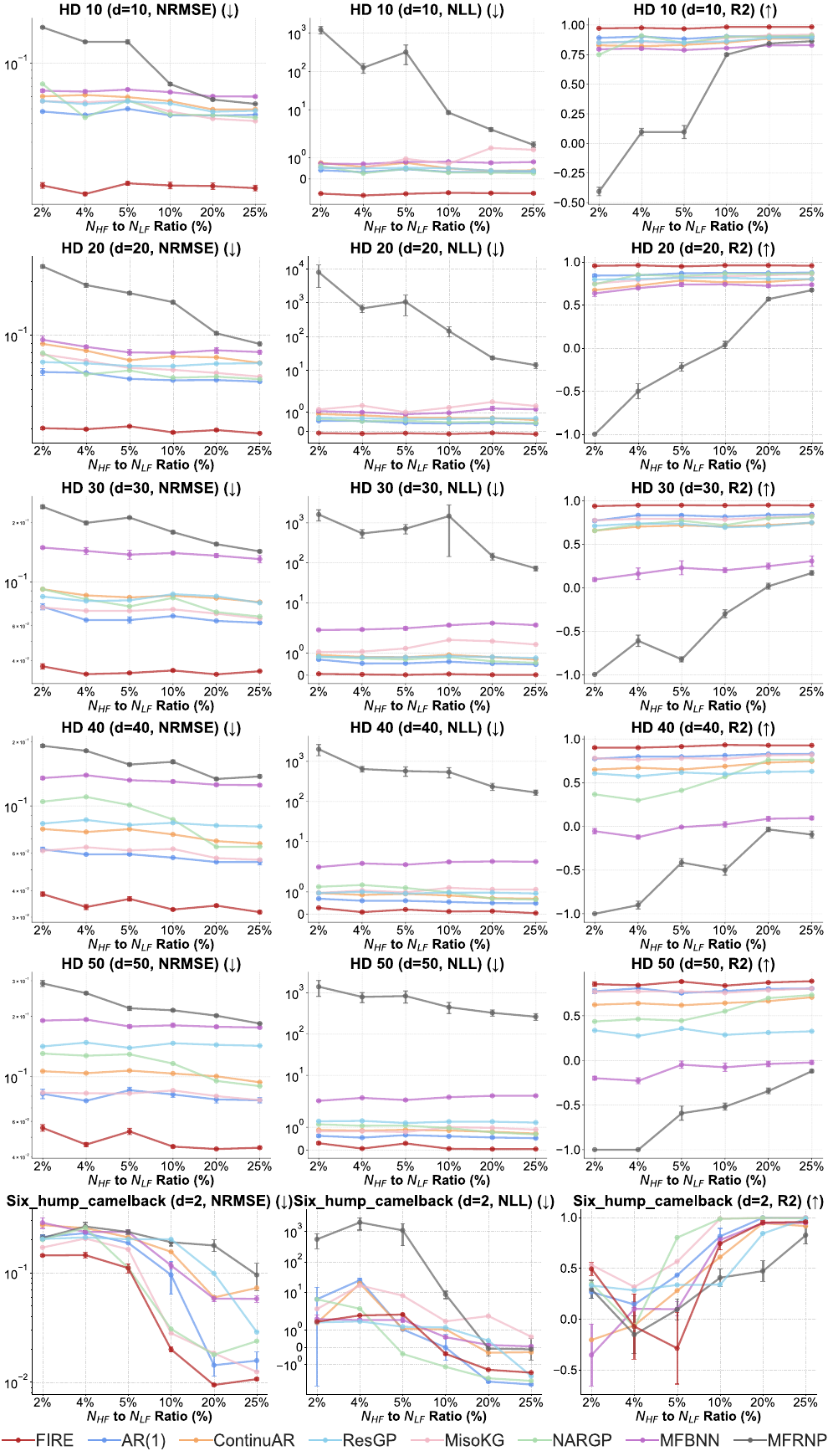}
  \caption{Result of synthetic problems (HD 10$\sim$HD 50, sixhump-camelback).}
  \label{fig:Synthetic_single_problem_v3}
\end{figure}

\end{document}